\documentclass{article}
\pdfoutput=1
\usepackage[numbers]{natbib}

\usepackage[final]{nips_2018}

\usepackage[utf8]{inputenc} 
\usepackage[T1]{fontenc}    
\usepackage{hyperref}       
\usepackage{url}            
\usepackage{booktabs}       
\usepackage{amsfonts}       
\usepackage{nicefrac}       
\usepackage{microtype}      
\usepackage{algorithm,algorithmic}
\usepackage{amsmath}
\usepackage[pdftex]{graphicx}

\usepackage{amsthm}

\newtheorem{theorem}{Theorem}
\newtheorem{hypothesis}{Hypothesis}

\title{Backprop-Q: Generalized Backpropagation for Stochastic Computation Graphs}

\author{
    Xiaoran Xu$^1$,\; Songpeng Zu$^1$,\; Yuan Zhang$^2$\thanks{Work done during the internship in Hulu},\; Hanning Zhou$^1$,\; Wei Feng$^1$ \\
  $^1$Hulu Innovation Lab, Beijing, China \\
  \texttt{\{xiaoran.xu, songpeng.zu, eric.zhou, wei.feng\}@hulu.com} \\
  $^2$School of Electronics Engineering and Computer Science, Peking University, Beijing, China \\
  \texttt{yuan.z@pku.edu.cn} \\
}
 
\begin{document}

\maketitle

\begin{abstract}
In real-world scenarios, it is appealing to learn a model carrying out stochastic operations internally, known as \emph{stochastic computation graphs} (SCGs), rather than learning a deterministic mapping. However, standard backpropagation is not applicable to SCGs. We attempt to address this issue from the angle of cost propagation, with local surrogate costs, called \emph{Q-functions}, constructed and learned for each stochastic node in an SCG. Then, the SCG can be trained based on these surrogate costs using standard backpropagation. We propose the entire framework as a solution to generalize backpropagation for SCGs, which resembles an actor-critic architecture but based on a graph. For broad applicability, we study a variety of SCG structures from one cost to multiple costs. We utilize recent advances in reinforcement learning (RL) and variational Bayes (VB), such as off-policy critic learning and unbiased-and-low-variance gradient estimation, and review them in the context of SCGs. The generalized backpropagation extends transported learning signals beyond gradients between stochastic nodes while preserving the benefit of backpropagating gradients through deterministic nodes. Experimental suggestions and concerns are listed to help design and test any specific model using this framework.
\end{abstract}
\section{Introduction}

The \emph{Credit assignment problem} has been seen as the fundamental learning problem. Given a long chain of neuron connections, it studies how to assign "credit" to early-stage neurons for their impact on final outcome through downstream connections. It dates back to a well-known approach, called \emph{backpropagation} \cite{Rumelhart1986_LearningRepByBackpropErr}. Error signals are propagated from the output layer to hidden layers, guiding weight updates. The "credit" is a signal of loss gradient calculated by the chain rule. Meanwhile, some work attempts to seek an alternative to exact gradient computation, either by finding a biologically plausible implementation \cite{Bengio2015_TowardsBioPlausibleDL,Lee2015_DifferenceTargetProp,Lillicrap2016_RandomSynapticFeedbackWeights,Nokland2016_DirectFeedbackAlignment}, or using synthetic gradients \cite{Jaderberg2016_DecoupledNeuralInterfaces}. The insight is that building feedback pathways may play a more crucial role than assuring the preciseness of propagated gradients. More specifically, instead of gradients, feedback learning signals can be target values \cite{Lee2015_DifferenceTargetProp}, synthetic gradients \cite{Jaderberg2016_DecoupledNeuralInterfaces}, or even signals carried through random feedback weights \cite{Lillicrap2016_RandomSynapticFeedbackWeights,Nokland2016_DirectFeedbackAlignment}.

However, the great success of deep neural networks in a variety of real-world scenarios is largely attributed to the standard gradient-based backpropagation algorithm due to its effectiveness, flexibility, and scalability. The major weakness is its strict requirement that neural networks must be deterministic and differentiable, with no stochastic operations permitted internally. This limits the potential of neural networks for modeling large complex stochastic systems. Therefore, rather than figuring out an alternative to backpropagation, we aim at extending it to become applicable to arbitrary stochastic computation graphs. Specifically, we propose to conduct the propagation process across stochastic nodes, with propagated learning signals beyond gradients, while preserving the benefit of standard backpropagation when transporting error gradients through the differentiable part.

Recently, many efforts have focused on solving the tasks that require effective training by backpropagation along with sampling operations, called \emph{backpropagation through stochastic neurons}. As one of the early work, \cite{Bengio2013_EstimatingOrPropagating} studied four families of solutions to estimate gradients for stochastic neurons, including the straight-through estimator, but limited to binary neurons. 

In variational inference and learning, training with samples arises from the fact that it optimizes an expectation-form objective, a variational lower bound, with respect to distribution parameters. Based on Monte Carlo sampling, several unbiased and low-variance estimators have been proposed for continuous and discrete random variables, using the techniques such as the reparameterization trick \cite{Kingma2014_VAE,Rezende2014_DLGM,Ruiz2016_GREP}, control variates \cite{Mnih2014_NVIL,Paisley2012_VBI,Ranganath2013_BBVI}, continuous relaxation \cite{Maddison2017_Concrete,Jang2017_GumbelSoftmax} and most recently, hybrid methods combining the previous techniques \cite{Gu2016_MuProp,Tucker2017_Rebar,Grathwohl2018_RELAX}. However, these methods studied a direct cost $f(z)$ defined on random variables, without systematically considering the effect of long-delayed costs after a series of stochastic operations, which is the key of the credit assignment problem. 

In reinforcement learning, a Markov decision process can be viewed as a chain of stochastic actions and states, and the goal is to maximize the expected total rewards, with delayed rewards considered. The temporal-difference (TD) learning method \cite{Sutton2017_RL}, along with policy gradient methods \cite{Williams1992,Sutton2000_PolicyGradient} and various on- and off-policy techniques, such as experience replay \cite{Mnih2015_HumanLevel,Lillicrap2016_DDPG,Gu2017_QProp}, separate target network \cite{Mnih2015_HumanLevel,Mnih2016_A3C,Lillicrap2016_DDPG,Gu2017_QProp}, advantage function \cite{Sutton2000_PolicyGradient,Schulman2016_GAE} and controlled policy optimization \cite{Schulman2015_TRPO,Schulman2017_PPO}, provide a powerful toolbox to solve temporal credit assignment \cite{suttons1984}. However, rare work has thought of reinforcement learning from a nonsequential perspective, for example, a more structured decision graph, made of a mix of policy networks, with various value functions interwoven and learned jointly.

The learning problem for SCGs was first clearly formulated in \cite{Schulman2015_SCG}, solved by a modification of standard backpropagation, much like a graph-based policy gradient method without critic learning. Inspired by this work, we study the possibility of backpropagating value-based signals in TD-style updates, as a complement to gradient-based signals, and propose a more generalized framework to implement backpropagation, called \emph{Backprop-Q}, applicable to arbitrary SCGs, absorbing many useful ideas and methods recently introduced in RL and VB. 

In this paper, our contributions mainly focus on two aspects: (1) cost propagation and (2) how to construct and learn local surrogate costs.  For cost propagation, to transport expectation of costs back through stochastic nodes, we introduce a Backprop-Q network associated with a set of tractable sample-based update rules. For local surrogate costs, we parameterize each by a neural network with compact input arguments, analogous to a critic (or a value function) in reinforcement learning. To the best of our knowledge, this paper is the first to consider learning critic-like functions from a graph-based view. Combined with standard backpropagation, our work depicts a big picture where feedback signals can go across stochastic nodes and go beyond gradients. 

The primary purpose of this paper is to provide a learning framework with wide applicability and offer a new path to training arbitrary models that can be represented in SCGs, at least formally. In practice, much future work needs to be done to examine what specific type of SCG problems can be solved effectively and what trick needs to be applied under this framework. Despite lack of experimental demonstration, we list possible suggestions and concerns to conduct future experiments.

\section{Preliminary}

\textbf{Stochastic computation graphs (SCGs).} We follow Schulman's \cite{Schulman2015_SCG} definition of SCGs, and represent an SCG as $(\mathcal{X},\mathcal{G}_\mathcal{X}, \mathcal{P}, \Theta, \mathcal{F}, \Phi)$. $\mathcal{X}$ is the set of random variables, $\mathcal{G}_\mathcal{X}$ the directed acyclic graph on $\mathcal{X}$, $\mathcal{P} = \{ p_{\scriptscriptstyle X} (\cdot | {Pa}_{\scriptscriptstyle X}; \theta_{\scriptscriptstyle X}) \mid X \in \mathcal{X}\}$ the set of conditional distribution functions parameterized by $\Theta$, and $\mathcal{F} = \{f_i(\mathcal{X}_i; \phi_i) \mid \mathcal{X}_i \subseteq \mathcal{X} \}$ the set of cost functions parameterized by $\Phi$. Although an SCG contains parameter nodes (including inputs), deterministic nodes and stochastic nodes, for simplicity we leave out notations for deterministic nodes as they are absorbed into functions $p_{\scriptscriptstyle X}$ or $f_i$. Note that $\mathcal{G}_\mathcal{X}$ and $\mathcal{P}$ make a probabilistic graphical model (PGM) such that the feedforward computation through SCG performs ancestral sampling. However, an SCG expresses different semantics from a PGM by $\mathcal{G}_\mathcal{X}$ in two aspects: 1) it contains costs; 2) the detail of deterministic nodes and their connections to other nodes reveals a finer modeling of computation dependencies omitted by $\mathcal{G}_\mathcal{X}$. Furthermore, due to the flexibility of expressing dependencies in SCGs, parameters can be shared or interacted across $\Theta, \Phi$ without being limited to local parameters.

\textbf{Learning problem in SCGs.} The learning problem in an SCG is formulated as minimizing an expected total cost $J(\Theta,\Phi) = \mathbb{E}_{\mathcal{X} \sim \mathcal{P};\Theta} [\sum f_i(\mathcal{X}_i; \phi_i)]$ over distribution parameters in $\Theta$ and cost parameters in $\Phi$ jointly. $J$ is usually intractable to compute and therefore approximated by Monte Carlo integration. When applying stochastic optimization, the stochasticity arises not only from mini-batch data but also from sampling procedure, resulting in imprecision and difficulty, compared to optimizing a deterministic neural network. However, SCGs apply to a much wider variety of tasks as long as their objective functions can be written in expectation.

\textbf{SCGs for probabilistic latent models.} For probabilistic latent models, the formulation using SCGs has two different ways: sampling via generative models or via inference networks. The former fits a latent-variable model by maximizing likelihood $p(x;\theta) = \mathbb{E}_{p(z;\theta)}[p(x|Z; \theta)]$ for a single observation. The latter is more popular, known as variational Bayes \cite{Mnih2014_NVIL,Kingma2014_VAE,Rezende2014_DLGM}. Here, the inference network acts as an SCG that performs actual sampling, and the generative model only provides probabilistic functions to help define a variational lower bound as the SCG's cost, with the approximate posterior as well. The final expected cost will be $\mathbb{E}_{q(z|x;\phi)}[\log p(Z;\theta) + \log p(x|Z;\theta) - \log q(Z|x;\phi)]$.

\textbf{SCGs for reinforcement learning.} SCGs can be viewed in the sense of reinforcement learning under known deterministic transition. For each stochastic node $X$, $p_{\scriptscriptstyle X}(x|{Pa}_{\scriptscriptstyle X}; \theta_{\scriptscriptstyle X})$ is a policy where $x$ means an action and ${Pa}_{\scriptscriptstyle X}$ means the state to take action $x$. Whenever an action is taken, it possibly becomes part of a state for next actions taken at downstream stochastic nodes. Although it simplifies reinforcement learning without considering environment dynamics, it is no longer a sequential decision process, but a complex graph-based decision making, which integrates various policies, with rewards or costs coming from whatever branch leading to a cost function.

\textbf{SCGs for stochastic RNNs.} Traditional RNNs build a deterministic mapping from inputs to predictions, resulting in exposure bias when modeling sentences \cite{Ranzato2016_SeqLevelTrain}. It is trained by ground truth words as oppose to words drawn from the model distribution. Using a stochastic RNN, an instance of SCGs, can overcome the issue, because the next word is sampled based on its previous words.
\section{Basic Framework of Backprop-Q}

In this section, we first demonstrate how to construct local surrogate costs and derive their update rules in one-cost SCGs. Then, we extend our methods to multi-cost SCGs with arbitrary structure.

\subsection{One-Cost SCGs}

\textbf{Cost propagation.} Why is cost propagation needed? If we optimize $\mathbb{E}_{{\scriptscriptstyle X} \sim p(\cdot;\theta)}[f(X)]$ over $\theta$, we can get an unbiased gradient estimator by applying the REINFORCE \cite{Williams1992} directly. However, considering a long chain with an objective $\mathbb{E}_{\scriptscriptstyle X_{1:t-1}}[ \mathbb{E}_{\scriptscriptstyle X_t|X_{1:t-1}} [ \mathbb{E}_{\scriptscriptstyle X_{t+1:T}|X_t} [f(X_{\scriptscriptstyle T})]]]$, a given $x_t$ is supposed to be associated with the conditional expected cost $\mathbb{E}_{{\scriptscriptstyle X_{t+1:T}}|x_t}[f(X_{\scriptscriptstyle T})]$, rather than a delayed $f(x_{\scriptscriptstyle T})$. The REINFORCE estimator is notorious for high variance due to the sampling-based approximation for $\mathbb{E}_{{\scriptscriptstyle X}_{t}|x_{1:t-1}}[\cdot]$ given $x_{1:t-1}$, and using $f(x_{\scriptscriptstyle T})$ after sampling over $X_{t+1:T}$ across a long chain will make it much worse. Unlike \cite{Schulman2015_SCG} without addressing this issue, we aim at learning expected costs conditioned on each random variable and using Rao-Blackwellization \cite{Casella1996} to reduce variance due to $\text{Var}(\mathbb{E}_{\scriptscriptstyle Y|X}[f(Y)]) \le \text{Var}(f(Y))$. We find that these expected costs follow a pattern of computing expectation updates on one random variable each time, starting from the cost and flowing backward through all random variables.

\textbf{Local surrogate costs.} In a chain-like SCG, cost propagation based on expectation updates resembles learning a value function in reinforcement learning, which is a function of current state or state-action pair. However, in a general SCG, the expected costs appear more complex.
\begin{theorem}
\textbf{(SCG's gradient estimators)} Given an SCG with a cost function $f$ defined on $\mathcal{Z} \subseteq \mathcal{X}$, and each random variable associated with its own distribution parameter such that $X\sim p(\cdot|{Pa}_{\scriptscriptstyle X}; \theta_{\scriptscriptstyle X})$, the gradient of the expected total cost $J$ with respect to $\theta_{\scriptscriptstyle X}$ can be written as:
\begin{equation} \label{eq-gradient}
\nabla_{\theta_X} J = \mathbb{E}_{{An}_X,X} \big[ \nabla_{\theta_{\scriptscriptstyle X}} \log{ p(X|{Pa}_{\scriptscriptstyle X};\theta_{\scriptscriptstyle X}) \cdot Q_{\scriptscriptstyle X}({Fr}_{{An}_X \cup \{X\}}) } \big]
\end{equation}
where ${Pa}_{\scriptscriptstyle X}$ is the set of $X$'s parents, ${An}_{\scriptscriptstyle X}$ the set of $X$'s ancestors and ${Fr}_{\scriptscriptstyle V}\subseteq V$ the frontier \footnote{In a multi-cost SCG, a cost $f$ must be specified for a frontier, denoted as ${Fr}^f_{\scriptscriptstyle V}$} of a set of random variables $V$, defined as: a subset of random variables from which the cost is reachable through random variables not in $V$. We also define a Q-function for each stochastic node, representing the expected cost depending on this random variable and its necessary ancestors such that:
\begin{equation}
Q_{\scriptscriptstyle X}({Fr}_{{An}_X \cup \{X\}}) := \mathbb{E}_{Z|{Fr}_{{An}_{X} \cup \{X\}}}[f(\mathcal{Z})]
\end{equation}
\end{theorem}
The Q-function $Q_{\scriptscriptstyle X}$ has an enlarged scope when a bypass goes around $X$ to the cost. The scope incorporates the ancestor of $X$ from which the bypass starts, carrying extra information needed at $X$ when evaluating $Q_{\scriptscriptstyle X}$. The scope thus makes a frontier set for $X$ and all its ancestors, indicating the Markov property that given this scope the remaining ancestors will not affect the cost. Therefore, $Q_{\scriptscriptstyle X}$ acts as a \emph{local surrogate cost} to $X$ of the remote cost, much like seeing what the future looks like from the perspective of its own scope and trying to minimize $\mathbb{E}_{{An}_X,X} [Q_{\scriptscriptstyle X}({Fr}_{{An}_X \cup \{X\}})]$.

\begin{figure*}
\centering
\includegraphics[width=\textwidth]{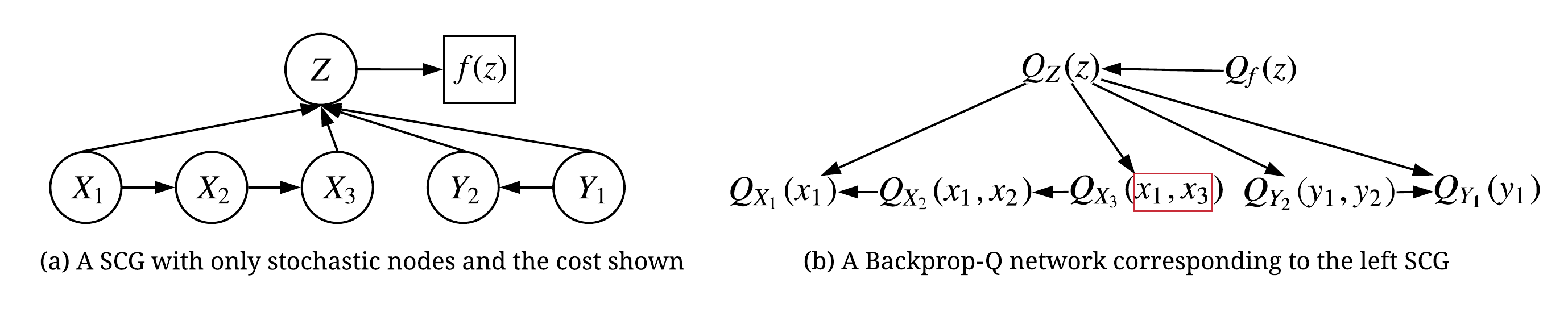}
\vspace{-20pt}
\caption{An instance of one-cost SCGs and its Backprop-Q network}
\vspace{-10pt}
\label{fig-BPQ-net-1c}
\end{figure*}

\textbf{Backprop-Q network.} To propagate cost, we need to derive the rules of expectation updates. Let $X,Y$ be two stochastic nodes such that $X\in {Pa}_{\scriptscriptstyle Y}$ and then we have: $Q_{\scriptscriptstyle X}({Sc}_{\scriptscriptstyle X}) = \mathbb{E}_{V|{Sc}_{X}}[Q_{\scriptscriptstyle Y}({Sc}_{\scriptscriptstyle Y})]$, where scope ${Sc}_{\scriptscriptstyle X} = {Fr}_{{An}_X \cup \{X\}}$, ${Sc}_{\scriptscriptstyle Y} = {Fr}_{{An}_Y \cup \{Y\}}$ and $V = {Sc}_{\scriptscriptstyle Y} - {An}_{\scriptscriptstyle X} \cup \{X\}$ represents what variables are still unknown in $Sc_{\scriptscriptstyle Y}$ at node $X$. Figure \ref{fig-BPQ-net-1c} shows that a Q-function may have more than one equivalent update rules, e.g., $Q_{\scriptscriptstyle X_1}$ and $Q_{\scriptscriptstyle Y_1}$, when a node has multiple paths reaching the cost. The update rules between Q-functions can be represented by the reversed $\mathcal{G}_\mathcal{X}$ of an SCG, plus the cost as a root. We call it a \emph{Backprop-Q network}. Each node in a Backprop-Q network is a Q-function\footnote{We consider a cost $f$ a special Q-function, deonted as $Q_f(\cdot) := f(\cdot)$ with the same scope as $f$.}, e.g., $Q_{\scriptscriptstyle X}({Sc}_{\scriptscriptstyle X})$, indexed by a stochastic node $X$ in $\mathcal{G}_\mathcal{X}$, a scope denoted as ${Sc}_{\scriptscriptstyle X}$ and a cost source\footnote{In the multi-cost setting, we need to label a cost source for Q-functions, e.g., $Q_{\scriptscriptstyle X}^f({Sc}_{\scriptscriptstyle X})$}. We represent a Backprop-Q network as $(\mathcal{Q},\mathcal{G}_\mathcal{Q},\mathcal{R})$, where $\mathcal{Q}$ is the set of Q-functions, $\mathcal{G}_\mathcal{Q}$ the directed acyclic graph on $\mathcal{Q}$ and $\mathcal{R} = \{ {\mathcal{R}}_{\scriptscriptstyle X} \mid {\mathcal{R}}_{\scriptscriptstyle X} Q_{\scriptscriptstyle X}({Sc}_{\scriptscriptstyle X}) := \mathbb{E}[Q_{\scriptscriptstyle Y}({Sc}_{\scriptscriptstyle Y})], X \in {Pa}_{\scriptscriptstyle Y}, \forall X \in \mathcal{X} \}$ the set of update-rule operators. If $Q_{\scriptscriptstyle X}$ has multiple equivalent update rules, we pick any one or take the average. In multi-cost SCGs, we will meet multiple $Q_{\scriptscriptstyle X}$ with different scopes and cost sources at the same node $X$, making $\mathcal{G}_\mathcal{Q}$ no more a reversed $\mathcal{G}_\mathcal{X}$.

\textbf{Learning local surrogate cost.} If a local surrogate cost is exactly a true expected cost, we can obtain an unbiased gradient estimator by Eq.\ref{eq-gradient}. However, computing a sweep of expectation updates is usually intractable. We thus turn to sample updates. For each Q-function, we sample one step forward, use this sample to query the next Q-function and then update it as: $Q_{\scriptscriptstyle X}({Sc}_{\scriptscriptstyle X}) \leftarrow  Q_{\scriptscriptstyle X}({Sc}_{\scriptscriptstyle X}) + \alpha [Q_{\scriptscriptstyle Y}(y,{Sc}_{\scriptscriptstyle Y}^{-y}) - Q_{\scriptscriptstyle X}({Sc}_{\scriptscriptstyle X})]$, where $y \sim p_{\scriptscriptstyle Y}(\cdot|{Pa}_{\scriptscriptstyle Y};\theta_{\scriptscriptstyle Y})$ is the drawn sample, assuming $X \in {Pa}_{\scriptscriptstyle Y}$ and other parents known, and $\alpha$ is a step size. We can also run an ancestral sampling pass and generate a full set of samples to then update each Q-function backward. It is a graph version of on-policy TD-style learning. The downside is sampling error and accumulated incorrectness of downstream Q-functions due to lack of exact expectation computation. Is there a convergence guarantee? Would these Q-functions converge to the true expected costs? In a tabular setting, the answer is yes as in reinforcement learning \cite{Sutton2017_RL}. When Q-functions are estimated by function approximators, denoted as $Q_w$, especially in a nonlinear form like neural networks, we have the convergence guarantee as well, so long as each Q-function is independently parameterized and trained sufficiently, as opposed to what we know in reinforcement learning. When learning $Q_{w_{X}}$ from $Q_{w_Y}$, for example, applying sample updates is actually doing one-step stochastic gradient descent to reduce the expected squared errors by optimizing $w_{\scriptscriptstyle X}$:
\begin{equation}
\begin{split}
{Err}(w_{\scriptscriptstyle X}) & := \mathbb{E}_{An_Y,Y}[(Q_{w_{X}}({Sc}_{\scriptscriptstyle X}) - Q_{w_{Y}}({Sc}_{\scriptscriptstyle Y}))^2] \\
& \ge \mathbb{E}_{An_X, X}[(Q_{w_{X}}({Sc}_{\scriptscriptstyle X}) - \mathbb{E}_{{Sc}_Y - {An}_X \cup \{X\} | {Sc}_X} [Q_{w_{Y}}({Sc}_{\scriptscriptstyle Y})])^2]
\end{split}
\end{equation}
The one-step update on $w_{\scriptscriptstyle X}$ is: $w_{\scriptscriptstyle X} \leftarrow w_{\scriptscriptstyle X} + \alpha (Q_{w_{Y}}({Sc}_{\scriptscriptstyle Y}) - Q_{w_{X}}({Sc}_{\scriptscriptstyle X})) \nabla Q_{w_{X}}({Sc_{\scriptscriptstyle X}})$.
\begin{theorem}
\textbf{(Convergence of learned Q-functions)} Given a Backprop-Q network with one cost as the root, if the expected squared error between each learned $Q_{w_X}$ and its parent $Q_{w_Y}$ can be bounded by $\epsilon$ ($\epsilon > 0$) such that $\mathbb{E}_{An_Y,Y}[(Q_{w_{X}}({Sc}_{\scriptscriptstyle X}) - Q_{w_{Y}}({Sc}_{\scriptscriptstyle Y}))^2] \le \epsilon$, then we have:
\begin{equation}
\mathbb{E}_{An_X,X}\big[\big(Q_{w_{X}}({Sc}_{\scriptscriptstyle X}) - Q_{\scriptscriptstyle X}({Sc}_{\scriptscriptstyle X})\big)^2\big] \le (3\cdot 2^{l_{Q_X}-1} - 2)\epsilon \quad \text{ for } l_{Q_X} \ge 1
\end{equation}
where $Q_{\scriptscriptstyle X}({Sc}_{\scriptscriptstyle X})$ represents the true expected cost and $l_{Q_X}$ the length of the path from $Q_{\scriptscriptstyle X}$ to the root.
\end{theorem}
The above shows the deviations from true Q-functions accumulate as $l_{Q_{X}}$ increases. As a Backprop-Q network has a finite size, the deviations can go infinitely small when each $Q_{w_X}$ is sufficiently trained to fit $Q_{w_Y}$. Due to independent parameterization, optimizing $w_{\scriptscriptstyle X}$ will not affect $Q_{w_Y}$'s convergence.

\textbf{SCGs with a multivariate cost.} For a cost defined on multiple random variables, e.g., $f(X_1,X_2,X_3)$, we can assume a virtual node prepended to the cost, which collects all the random variables in $f$'s scope into one big random variable $\mathcal{Z} = (X_1,X_2,X_3)$, following a deterministic conditional distribution $\mathcal{Z} \sim p_{\scriptscriptstyle \mathcal{Z}}(\cdot|X_1,X_2,X_3)$. The rest procedure is the same as the above.

\textbf{SCGs with shared parameters.} Consider a case with parameter $\theta$ shared by all distributions and even the cost. We replace $\theta$ with local parameters, e.g., $\theta_{\scriptscriptstyle X}$, each only corresponding to one random variable but constrained by identity mapping $\theta_{\scriptscriptstyle X} = \theta$. To compute $\nabla_{\theta} J$, we compute the gradients w.r.t. each local parameter and then take the sum of them as $\nabla_{\theta} J = \sum_{\scriptscriptstyle X} \nabla_{\theta_X} J$.

\textbf{Remarks.} Standard backpropagation transports gradients, the first-order signals, while we propagate the zero-order signals of function outputs through stochastic nodes, with cumulative effect by past updates. When the approximate Q-functions get close to the true ones, we can expect their first-order derivatives also get close to the true gradients in some sense, which means we can utilize the gradients of the approximate Q-functions as well. The theoretic analysis can be found in Appendix.

\subsection{Multi-Cost SCGs}

\begin{figure*}
\centering
\includegraphics[width=\textwidth]{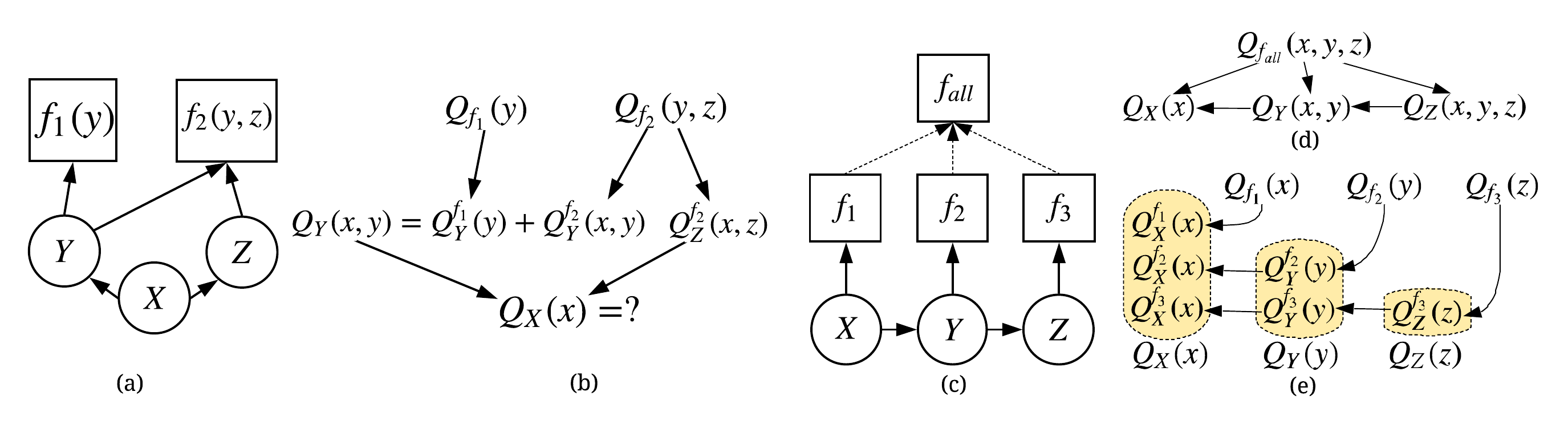}
\vspace{-20pt}
\caption{Multi-cost SCGs and their Backprop-Q networks}
\vspace{-10pt}
\label{fig-mc-a}
\end{figure*}

\textbf{Trouble caused by multiple costs.} A stochastic node leading to multiple costs, e.g., $Y$ in Figure \ref{fig-mc-a}(a), may have Q-functions of different scopes and different cost sources as shown in Figure \ref{fig-mc-a}(b). The expected cost $Q_{\scriptscriptstyle Y}(x,y)$ at node $Y$ is the sum of those from two costs respectively. However, it is confusing to update $Q_{\scriptscriptstyle X}(x)$ based on $Q_{\scriptscriptstyle Y}(x,y)$ and $Q_{\scriptscriptstyle Z}^{f_2}(x,z)$, because summing will double $f_2$ and averaging will halve $f_1$. We can treat two costs separately to build a Backprop-Q network for each, so that we can track cost sources and take the update target for $Q_{\scriptscriptstyle X}$ as: $Q_{\scriptscriptstyle Y}^{f_1}(y) + (Q_{\scriptscriptstyle Y}^{f_2}(x,y) + Q_{\scriptscriptstyle Z}^{f_2}(x,z))/2$. However, it is expensive to build and store a separate Backprop-Q network for each cost, and maintain probably multiple Q-functions at one stochastic node. An alternative way is to wrap all costs into one and treat it as a one-cost case as shown in Figure \ref{fig-mc-a}(c), but the scopes of Q-functions can be lengthy as in Figure \ref{fig-mc-a}(d).

\textbf{Multi-cost Backprop-Q networks.} In many cases, per-cost Backprop-Q networks can be merged and reduced. For example, in Figure \ref{fig-mc-a}(e), we sum Q-functions at each stochastic node into one, i.e., $Q_{\scriptscriptstyle X}(x) := Q_{\scriptscriptstyle X}^{f_1}(x) + Q_{\scriptscriptstyle X}^{f_2}(x) + Q_{\scriptscriptstyle X}^{f_3}(x)$, thus requiring only one Q-function at node $X$. The process resembles the one-step TD method in reinforcement learning, except that Q-functions are parameterized independently.
\begin{theorem}
\textbf{(Merging Backprop-Q networks)} Two Backprop-Q networks can be merged at stochastic node $X$ and its ancestors, if the two are fully matched from $X$ through $X$'s ancestors, that is, the set of the incoming edges to each ancestor in a Backprop-Q network is exactly matched to the other.
\end{theorem}
In Figure \ref{fig-mc-b} two costs provide separate Backprop-Q networks as in Figure \ref{fig-mc-b}(b). We can merge them at the last two nodes according to the above theorem. The update rules are always averaging all or picking one over incoming edges with the same cost source, and then summing those from different cost sources. Furthermore, we can reduce each Backprop-Q network into a directed rooted spanning tree, ensuring that each node receives exactly one copy of the cost. Many ways exist to construct a tree. Figure \ref{fig-mc-b}(c) shows a version with shorter paths but no benefit for merging, while Figure \ref{fig-mc-b}(d) constructs a chain version so that we can get a much more simplified Backprop-Q network.  

\textbf{Some complex cases.} The above merging guideline can apply to more complex SCGs, and result in a surprisingly reduced Backprop-Q network. In Appendix, we consider a stack of fully-connected stochastic layers, with costs defined on each stochastic node.

\begin{figure*}
\centering
\includegraphics[width=\textwidth]{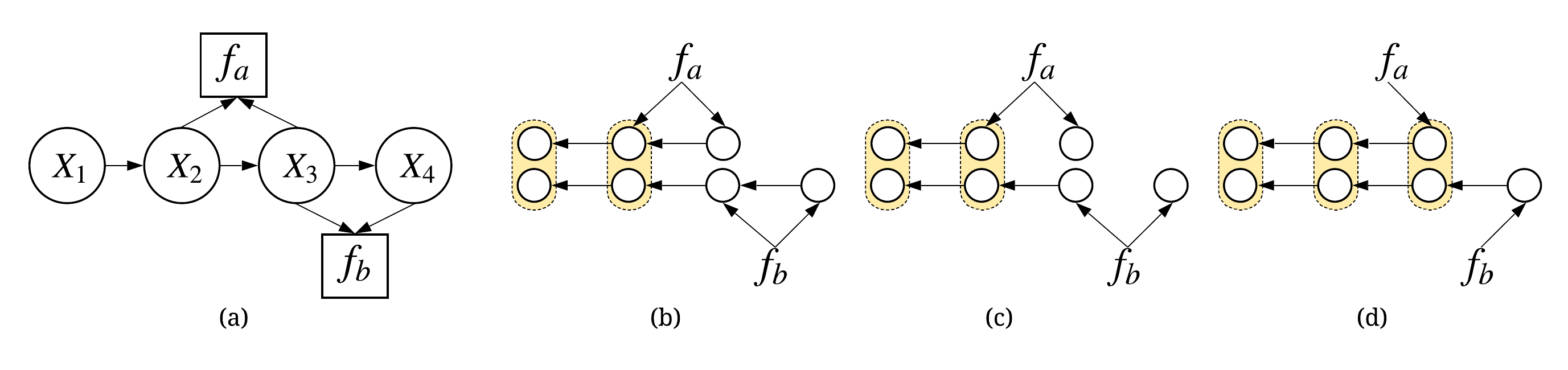}
\vspace{-20pt}
\caption{Merging two Backprop-Q networks}
\vspace{-10pt}
\label{fig-mc-b}
\end{figure*}
\section{Enhanced Backprop-Q}

\subsection{Using Techniques from Reinforcement Learning}

\textbf{$\lambda$-return updates.} $\lambda$-return provides a way of moving smoothly between Monte Carlo and one-step TD methods \cite{Sutton2017_RL}. It offers a return-based update, averaging all the $n$-step updates, each weighted proportional to $\lambda^{n-1}$. If $\lambda=1$, it gives a Monte Carlo return;  if $\lambda=0$, it reduces to the one-step return. Therefore, $\lambda$ trades off estimation bias with sample variance. We borrow the idea from \cite{Sutton2017_RL,Schulman2016_GAE} to derive a graph-based $\lambda$-return method. For each node, we collect upstream errors, multiplied by $\lambda$ and the discount factor $\gamma$, and add it to the current TD error. The combined error then propagates downstream. It follows the same pattern (averaging or summing) as the update rules defined by Backprop-Q networks. Cost propagation turns into propagation of TD errors. The limitation is that the updating must run in a backward pass synchronously. Some cases can be found in Appendix.

\textbf{Experience replay.} This technique is used to avoid divergence when training large neural networks \cite{Mnih2015_HumanLevel,Lillicrap2016_DDPG,Gu2017_QProp}. It keeps the recent $N$ experiences in a replay buffer, and applies TD updates to experience samples drawn uniformly at random. It breaks up a forward pass of ancestral sampling and may lose a full return. However, by reusing off-policy data, it breaks the correlation between consecutive updates and increases sample efficiency. It also allows asynchronous updating, which means cost propagation over a Backprop-Q network can be implemented at each node asynchronously. In a MDP, an experience tuple is $(s_t,a_t,s_{t+1})$ and then a sample $a_{t+1}$ is drawn by a target policy. In the setting of SCGs, we develop a graph-based experience replay that an experience tuple for node $Q_{\scriptscriptstyle X}$ is represented as $(X, \mathcal{A}, \mathcal{B}^1, \mathcal{A}^1 \ldots, \mathcal{B}^K, \mathcal{A}^K)$, where $\mathcal{A}$ is $X$'s ancestors in $Q_{\scriptscriptstyle X}$'s scope, $\mathcal{B}^k$ represents other potential parents affecting a common child $Y_k$ with $X$, and $\mathcal{A}^k$ is $Y_k$'s ancestors in $Q_{{\scriptscriptstyle Y}_k}$'s scope. Here, we assume that $X$ has $K$ children, which means that $Q_{\scriptscriptstyle X}$ probably has $K$ upstream Q-functions to combine. The updates are based on the optimization given below:
\begin{displaymath}
\min_{w_{X}} \mathbb{E}_{(X,\mathcal{A},\mathcal{B}^1,\mathcal{A}^1,\ldots,\mathcal{B}^K,\mathcal{A}^K)\sim \text{Uniform}(RB)} \mathbb{E}_{^{Y_k\sim p(\cdot|X,\mathcal{B}^k;\theta_{Y_k})}_{k=1,\ldots,K}} \Big[ \Big( \sum_{k=1}^K Q_{w_{Y_k}}(Y_k,\mathcal{A}^k) - Q_{w_X}(X,\mathcal{A}) \Big)^2 \Big]
\end{displaymath}
where ${RB}$ means a replay buffer. A case can be found in Appendix as an illustration.

\textbf{Other techniques.} (1) To improve stability and avoid divergence, we borrow the ideas from \cite{Mnih2015_HumanLevel,Lillicrap2016_DDPG} to develop a \emph{slow-tracking target} network. (2) We study graph-based \emph{advantage functions} and use them to replace Q-functions in the gradient estimator to reduce variance. (3) We apply \emph{controlled policy optimization} to distribution parameters in SCGs, using the ideas from \cite{Schulman2015_TRPO,Schulman2017_PPO}. See Appendix. 

\subsection{Using Techniques from Variational Bayesian Methods} In the framework of generalized backpropagation, after learning local surrogate costs for stochastic nodes, we need to train distribution parameters of the SCG, that is, we should continue the backpropagation process to transport gradients of local costs through underlying differentiable subgraphs. However, there is still one obstacle we must overcome. The objective function, $\mathbb{E}_{{\scriptscriptstyle Z} \sim p(\cdot;\theta)}[f(Z)]$ where $f(z):=Q_{w_Z}(z)$, is an expectation in terms of a distribution we need to optimize. 

Stochastic optimization has been widely used to solve the optimization problem. The key is to obtain a low-variance and unbiased gradient estimator applicable to both continuous and discrete random variables. The simplest and most general method is the REINFORCE estimator \cite{Williams1992}, but it is usually impractical due to high variance. Recently, to solve the backpropagation through stochastic operations in variational inference and learning, several advanced methods have been proposed, including the reparameterization trick \cite{Kingma2014_VAE,Rezende2014_DLGM,Ruiz2016_GREP}, control variates \cite{Mnih2014_NVIL,Paisley2012_VBI,Ranganath2013_BBVI,Gu2016_MuProp}, continuous relaxation \cite{Maddison2017_Concrete,Jang2017_GumbelSoftmax} and some hybrid methods like Rebar \cite{Tucker2017_Rebar} and RELAX \cite{Grathwohl2018_RELAX} to further reduce variance and keep the gradient estimator unbiased. In Appendix, we illustrate these methods in SCGs. We find that the crux of the matter is to open up a differentiable path from parameters to costs or surrogate objectives. It is better to utilize gradient information, even approximate, rather than a function output. All the mentioned techniques can be applied to our learned Q-functions.

\section{The Big Picture of Backpropagation}

Looking over the panorama of learning in an SCG, we see that the Backprop-Q framework extends backpropagation to a more general level, propagating learning signals not only across deterministic nodes but also stochastic nodes. The stochastic nodes act like repeaters, sending expected costs back through all random variables. Then, each local parameterized distribution, which is a computation subgraph consisting of many deterministic and differentiable operations, takes over the job of backpropgation and then the standard backpropgation starts. Note that these computation subgraphs can overlap by sharing parameters with each other. See an illustration in Appendix. 
\section{Experimental Suggestions and Concerns}

SCGs can express a wide range of models in stochastic neural networks, VB and RL, which differ significantly. We provide experimental suggestions and concerns from three aspects listed below: 

\textbf{(1)} \emph{Choose a model to train by Backprop-Q with awareness of properties of the cost, graph structure, and types of random variables.} i) Is the cost differentiable? Does it involve SCG's distribution functions or parameters? Can it be decoupled and split into smaller costs? For example, think of the ELBO optimized in variational inference, and compare it with the discrete metric BLEU used in machine translation. ii) Does the graph contain long statistical dependencies? Does it hold only long-delayed costs, or have immediate costs? If the graph structure is flat and the delayed effect is weak, it might be better to use the MC-based actual cost value rather than that bootstrapped from learned Q-functions. iii) Is a random variable continuous or discrete? We suggest using the reparameterization trick for continuous variables if the probability is computable after transformation.

\textbf{(2)} \emph{Consider the way to learn Q-functions and how the trained SCG model might be impacted by the bias and inaccuracy of learned Q-functions.} i) Linear approximators converge fast and behave consistently, but cannot fit highly nonlinear functions, resulting in large bias. Nonlinear approximators based on neural networks can be unstable and hard to train, probably with higher sample complexity than using actual returns. ii) The policy gradient theorem \cite{Sutton2000_PolicyGradient} suggests using compatible features shared by policy and critic. We speculate that this might be related to the underlying factor of how Q-functions impact the SCG model, that is, how good of teaching signals Q-functions can offer might be more important than how well they fit the exact expected costs. iii) The sample updates to fit Q-functions may be correlated similarly to RL. We consider using experience replay and separate target networks to smooth data distribution for training Q-functions.

\textbf{(3)} \emph{Consider the way to utilize Q-functions.} i) A simple implementation is to treat a Q-function as a local cost, yielding a low-variance gradient estimator by applying one of the methods proposed in VB. However, the estimator is always biased, relying on how well the Q-function approximates to the exact expected cost. ii) We can treat a Q-function as a control variate to reduce the variance caused by actual returns, and correct the bias by a differentiable term based on this Q-function. See Appendix.
\section{Related Work}

\vspace{-1 pt}

Schulman \cite{Schulman2015_SCG} introduced a framework of automatic differentiation through standard backpropagation in the context of SCGs. Inspired by this work, we conduct a comprehensive survey from three areas. 

\textbf{Backpropagation-related learning:} The backpropagation algorithm, proposed in \cite{Rumelhart1986_LearningRepByBackpropErr}, can be viewed as a way to address the credit assignment problem, where the "credit" is represented by a signal of back-propagated gradient. Instead of gradients, people studied other forms of learning signals and other ways to assign them. \cite{Bengio2015_TowardsBioPlausibleDL,Lee2015_DifferenceTargetProp} compute targets rather than gradients using a local denoising auto-encoder at each layer and building both feedforward and feedback pathways. \cite{Lillicrap2016_RandomSynapticFeedbackWeights,Nokland2016_DirectFeedbackAlignment} show that even random feedback weights can deliver useful learning signals to preceding layers, offering a biologically plausible learning mechanism. \cite{Jaderberg2016_DecoupledNeuralInterfaces} uses synthetic gradients as error signals to work with backpropagation and update independently and asynchronously.

\textbf{Policy gradient and critic learning in RL:} Policy gradient methods offer stability but suffer from high variance and slow learning, while TD-style critic (or value function) learning are sample-efficient but biased and sometimes nonconvergent. Much work has been put to address such issues. For policy gradient, people introduced a variety of approaches, including controlled policy optimization (TRPO,PPO) \cite{Schulman2015_TRPO,Schulman2017_PPO}, deterministic policy gradient (DPG,DDPG) \cite{Silver2014_DPG,Lillicrap2016_DDPG}, generalized advantage estimation (GAE) \cite{Schulman2016_GAE} and control variates (Q-Prop) \cite{Gu2017_QProp}. TRPO and PPO constrain the change in the policy to avoid an excessively large policy update. DPG and DDPG, based on off-policy actor-critic, enable the policy to utilize gradient information from action-value function for continuous cases. GAE generalizes the advantage policy gradient estimator, analogous to TD($\lambda$). For critic learning, people focused mainly on the use of value function approximators \cite{Sutton2000_PolicyGradient} instead of actual returns \cite{Williams1992}. The TD learning \cite{Sutton2017_RL}, aided by past experience, is used widely. When using large neural networks to train action-value functions, DQN \cite{Mnih2015_HumanLevel} uses experience replay and a separate target network to break correlation of updates. A3C \cite{Mnih2016_A3C} proposes an asynchronous variant of actor-critic using a shared and slow-changing target network without experience replay. DDPG and Q-Prop inherit these two techniques to exploit off-policy samples fully and gain sample efficiency and model consistency.

\textbf{Gradient estimators in VB:} The problem of optimizing over distribution parameters has been studied a lot in VB. The goal is to obtain an unbiased and lower-variance gradient estimator. The basic estimator is the REINFORCE \cite{Williams1992}, also known as the score-function \cite{Fu2006} or likelihood-ratio estimator \cite{Glynn1990}. With the widest applicability, it is unbiased but suffer from high variance. The main approaches for variance reduction are the reparameterization trick \cite{Kingma2014_VAE,Rezende2014_DLGM,Ruiz2016_GREP} and control variates \cite{Mnih2014_NVIL,Paisley2012_VBI,Ranganath2013_BBVI}. The former is applicable to continuous variables with a differentiable cost. It is typically used with Gaussian distribution \cite{Kingma2014_VAE,Rezende2014_DLGM}. \cite{Ruiz2016_GREP} proposed a generalized reparameterization gradient for a wider class of distributions. To reparameterize discrete variables, \cite{Maddison2017_Concrete,Jang2017_GumbelSoftmax} introduced the Concrete and the Gumbel-Softmax distribution respectively to build relaxed models but bringing in bias. The latter is suitable for both continuous and discrete variables. A control variate can be an input-dependent term, known as baseline \cite{Mnih2014_NVIL}, or a sample-dependent term with an analytic expectation \cite{Paisley2012_VBI,Ranganath2013_BBVI}. It may obtain higher variance than the former in practice, intuitively because it cannot utilize gradient information of the cost but an outcome. Other variance reduction methods include local expectation gradients \cite{Titsias2015_LocalExpectationGradients} and straight-through estimator \cite{Bengio2013_EstimatingOrPropagating}. Recently, new advanced estimators have been proposed with lower variance and being unbiased. MuProp \cite{Gu2016_MuProp} uses the first-order Taylor expansion as a control variate, leaving the deterministic term computed by a mean-field network. Its model-free version for RL, Q-Prop \cite{Gu2017_QProp}, uses the similar technique combined with off-policy critic-learning by experience replay. Rebar \cite{Tucker2017_Rebar} and RELAX \cite{Grathwohl2018_RELAX} aim at deriving estimators for discrete variables. Unlike Rebar, RELAX learns a free-form control variate parameterized by a neural network.

\section{Conclusion}

\vspace{-1 pt}

In this paper, we propose a framework of generalized backpropagation for arbitrary stochastic computation graphs, enabling propagated signals to go across stochasticity and beyond gradients. 


\bibliography{reference}{}
\bibliographystyle{unsrt}

\newpage

\setcounter{section}{0}
\begin{centering}
{\Large{\textbf{Appendix}}}
\end{centering}

\section{Proofs}

\begin{theorem}
\textbf{(SCG's gradient estimators)} Given an SCG with a cost function $f$ defined on $\mathcal{Z} \subseteq \mathcal{X}$, and each random variable associated with its own distribution parameter such that $X\sim p(\cdot|{Pa}_{\scriptscriptstyle X}; \theta_{\scriptscriptstyle X})$, the gradient of the expected total cost $J$ with respect to $\theta_{\scriptscriptstyle X}$ can be written as:
\begin{equation}
\nabla_{\theta_X} J = \mathbb{E}_{{An}_X,X} \big[ \nabla_{\theta_{\scriptscriptstyle X}} \log{ p(X|{Pa}_{\scriptscriptstyle X};\theta_{\scriptscriptstyle X}) \cdot Q_{\scriptscriptstyle X}({Fr}_{{An}_X \cup \{X\}}) } \big]
\end{equation}
where ${Pa}_{\scriptscriptstyle X}$ is the set of $X$'s parents, ${An}_{\scriptscriptstyle X}$ the set of $X$'s ancestors and ${Fr}_{\scriptscriptstyle V}\subseteq V$ the frontier \footnote{In a multi-cost SCG, a cost $f$ must be specified for a frontier, denoted as ${Fr}^f_{\scriptscriptstyle V}$} of a set of random variables $V$, defined as: a subset of random variables from which the cost is reachable through random variables not in $V$. We also define a Q-function for each stochastic node, representing the expected cost depending on this random variable and its necessary ancestors such that:
\begin{equation}
Q_{\scriptscriptstyle X}({Fr}_{{An}_X \cup \{X\}}) := \mathbb{E}_{Z|{Fr}_{{An}_{X} \cup \{X\}}}[f(\mathcal{Z})]
\end{equation}
\end{theorem}
\begin{proof}
First, we rewrite the objective function for the SCG $(\mathcal{X},\mathcal{G}_\mathcal{X},\mathcal{P},\Theta,\mathcal{F},\Phi)$ by unfolding the whole expectation computation and splitting it into three parts, with the expectation on $X$ in the middle, as:
\begin{equation}
\begin{split}
J(\Theta,\Phi) & = \mathbb{E}_{{An}_X} \Big[ \mathbb{E}_{X|{Pa}_X} \big[ \mathbb{E}_{\mathcal{Z}|{An}_X \cup \{X\}} [f(\mathcal{Z})] \big] \Big] \\
& = \mathbb{E}_{{An}_X} \Big[ \mathbb{E}_{X|{Pa}_X} \big[ \mathbb{E}_{\mathcal{Z}|{{Fr}_{{An}_X \cup \{X\}}}} [f(\mathcal{Z})] \big] \Big]
\end{split}
\end{equation}
The second line follows the Markov property that given the frontier set ${Fr}_{{An}_X \cup \{X\}}$ the rest ancestors of $X$ have no impact on the cost. Then, we write the conditional distribution function $p(x|{Pa}_{\scriptscriptstyle X}; \theta_{\scriptscriptstyle X})$ explicitly in $J$:
\begin{equation}
J(\Theta,\Phi) = \mathbb{E}_{{An}_X} \Big[ \sum_{x} p(x|{Pa}_{\scriptscriptstyle X};\theta_{\scriptscriptstyle X}) \cdot \big[ \mathbb{E}_{\mathcal{Z}|{{Fr}_{{An}_X \cup \{x\}}}} [f(\mathcal{Z})] \big] \Big]
\end{equation}
Note that we can change the sum into an integral for the continuous case. Then, we can derive the gradient of $J$ with respect to the distribution parameter $\theta_{\scriptscriptstyle X}$ as follows:
\begin{equation}
\begin{split}
\nabla_{\theta_X} J(\Theta,\Phi) & = \mathbb{E}_{{An}_X} \Big[ \sum_{x} \nabla_{\theta_X} p(x|{Pa}_{\scriptscriptstyle X};\theta_{\scriptscriptstyle X}) \cdot \mathbb{E}_{\mathcal{Z}|{{Fr}_{{An}_X \cup \{x\}}}} [f(\mathcal{Z})] \Big] \\
& = \mathbb{E}_{{An}_X} \Big[ \sum_{x} p(x|{Pa}_{\scriptscriptstyle X};\theta_{\scriptscriptstyle X}) \nabla_{\theta_X} \log p(x|{Pa}_{\scriptscriptstyle X};\theta_{\scriptscriptstyle X}) \cdot \mathbb{E}_{\mathcal{Z}|{{Fr}_{{An}_X \cup \{x\}}}} [f(\mathcal{Z})] \Big] \\
& = \mathbb{E}_{{An}_X} \Big[ \mathbb{E}_{X|{Pa}_X} \big[ \nabla_{\theta_X} \log p(x|{Pa}_{\scriptscriptstyle X};\theta_{\scriptscriptstyle X}) \cdot \mathbb{E}_{\mathcal{Z}|{{Fr}_{{An}_X \cup \{X\}}}} [f(\mathcal{Z})] \big] \Big] \\
& = \mathbb{E}_{{An}_X,X} \big[ \nabla_{\theta_X} \log p(x|{Pa}_{\scriptscriptstyle X};\theta_{\scriptscriptstyle X}) \cdot \mathbb{E}_{\mathcal{Z}|{{Fr}_{{An}_X \cup \{X\}}}} [f(\mathcal{Z})] \big]
\end{split}
\end{equation}
The above result is an instance of the REINFORCE \cite{Williams1992} when we apply it to stochastic nodes in an SCG.
\end{proof}

\begin{theorem}
\textbf{(Convergence of learned Q-functions)} Given a Backprop-Q network with one cost as the root, if the expected squared error between each learned $Q_{w_X}$ and its parent $Q_{w_Y}$ can be bounded by $\epsilon$ ($\epsilon > 0$) such that $\mathbb{E}_{An_Y,Y}[(Q_{w_{X}}({Sc}_{\scriptscriptstyle X}) - Q_{w_{Y}}({Sc}_{\scriptscriptstyle Y}))^2] \le \epsilon$, then we have:
\begin{equation}
\mathbb{E}_{An_X,X}\big[\big(Q_{w_{X}}({Sc}_{\scriptscriptstyle X}) - Q_{\scriptscriptstyle X}({Sc}_{\scriptscriptstyle X})\big)^2\big] \le (3\cdot 2^{l_{Q_X}-1} - 2)\epsilon \quad \text{ for } l_{Q_X} \ge 1
\end{equation}
where $Q_{\scriptscriptstyle X}({Sc}_{\scriptscriptstyle X})$ represents the true expected cost and $l_{Q_X}$ the length of the path from $Q_{\scriptscriptstyle X}$ to the root.
\end{theorem}
\begin{proof}
Let $Q_f(\mathcal{Z}):=f(\mathcal{Z})$ be the root, where $f(\mathcal{Z})$ is a cost function defined on random variables $\mathcal{Z}$. For node $Q_{\scriptscriptstyle X}$, we can find a path from it to $Q_f$, denoted as $(Q_{{\scriptscriptstyle X}^{l}},Q_{{\scriptscriptstyle X}^{l-1}},\ldots,Q_{{\scriptscriptstyle X}^{0}})$ where $l$ is the length of the path, $Q_{{\scriptscriptstyle X}^{l}} := Q_{\scriptscriptstyle X}$ and $Q_{{\scriptscriptstyle X}^{0}} := Q_{f}$. Since we know
\begin{equation}
Q_{{\scriptscriptstyle X}^i}({Sc}_{{\scriptscriptstyle X}^i}) = \mathbb{E}_{V^{i-1}|{Sc}_{X^i}} \big[ Q_{{\scriptscriptstyle X}^{i-1}}({Sc}_{{\scriptscriptstyle X}^{i-1}}) \big] = \mathbb{E}_{V^{i-1}|{An}_{X^i},X^i} \big[ Q_{{\scriptscriptstyle X}^{i-1}}({Sc}_{{\scriptscriptstyle X}^{i-1}}) \big]
\end{equation}
where $V^{i-1} = {Sc}_{{\scriptscriptstyle X}^{i-1}} - An_{{\scriptscriptstyle X}^i} \cup \{X^i\}$, we thus can derive the following inequalities:
\begin{equation}
\begin{split}
& \mathbb{E}_{An_{{X}^i},X^i} \Big[\Big(Q_{w_{{X}^i}}({Sc}_{{\scriptscriptstyle X}^i}) - Q_{{\scriptscriptstyle X}^i}({Sc}_{{\scriptscriptstyle X}^i})\Big)^2 \Big] \\
= & \mathbb{E}_{An_{{X}^i},X^i} \Big[\Big(Q_{w_{{X}^i}}({Sc}_{{\scriptscriptstyle X}^i}) - \mathbb{E}_{V^{i-1}|An_{{X}^i},X^i} \big[ Q_{w_{{X}^{i-1}}}({Sc}_{{\scriptscriptstyle X}^{i-1}}) \big] \\
& \qquad \qquad + \mathbb{E}_{V^{i-1}|An_{{X}^i},X^i} \big[ Q_{w_{{X}^{i-1}}}({Sc}_{{\scriptscriptstyle X}^{i-1}}) \big] - \mathbb{E}_{V^{i-1}|{An}_{X^i},X^i} \big[ Q_{{\scriptscriptstyle X}^{i-1}}({Sc}_{{\scriptscriptstyle X}^{i-1}}) \big]\Big)^2 \Big] \\
\le & \mathbb{E}_{An_{{X}^i},X^i} \Big[ 2\Big(Q_{w_{{X}^i}}({Sc}_{{\scriptscriptstyle X}^i}) - \mathbb{E}_{V^{i-1}|An_{{X}^i},X^i} \big[ Q_{w_{{X}^{i-1}}}({Sc}_{{\scriptscriptstyle X}^{i-1}}) \big]\Big)^2 \\
& \qquad \qquad + 2\Big(\mathbb{E}_{V^{i-1}|An_{{X}^i},X^i} \big[ Q_{w_{{X}^{i-1}}}({Sc}_{{\scriptscriptstyle X}^{i-1}}) - Q_{{\scriptscriptstyle X}^{i-1}}({Sc}_{{\scriptscriptstyle X}^{i-1}}) \big]\Big)^2 \Big] \\
\le & 2 \mathbb{E}_{An_{{X}^{i-1}},X^{i-1}} \Big[ \Big(Q_{w_{{X}^i}}({Sc}_{{\scriptscriptstyle X}^i}) - Q_{w_{{X}^{i-1}}}({Sc}_{{\scriptscriptstyle X}^{i-1}}) \Big)^2 \Big] \\
& + 2 \mathbb{E}_{An_{{X}^{i-1}},X^{i-1}} \Big[ \Big( Q_{w_{{X}^{i-1}}}({Sc}_{{\scriptscriptstyle X}^{i-1}}) - Q_{{\scriptscriptstyle X}^{i-1}}({Sc}_{{\scriptscriptstyle X}^{i-1}})\Big)^2 \Big] \\
\le & 2\epsilon + 2 \mathbb{E}_{An_{{X}^{i-1}},X^{i-1}} \Big[ \Big( Q_{w_{{X}^{i-1}}}({Sc}_{{\scriptscriptstyle X}^{i-1}}) - Q_{{\scriptscriptstyle X}^{i-1}}({Sc}_{{\scriptscriptstyle X}^{i-1}})\Big)^2 \Big]
\end{split}
\end{equation}
Let $h_i = \mathbb{E}_{An_{{X}^{i}},X^{i}} \big[ \big( Q_{w_{{X}^{i}}}({Sc}_{{\scriptscriptstyle X}^{i}}) - Q_{{\scriptscriptstyle X}^{i}}({Sc}_{{\scriptscriptstyle X}^{i}})\big)^2 \big]$, indicating the deviation of a learned Q-function from the true one , and then we have the recursion inequality: $h_i \le 2 h_{i-1} + 2 \epsilon$. Because we do not need to approximate the root $Q_f$ which is explicitly known as $f$, we have $h_0 = 0$. Therefore, we can eventually obtain: $h_i \le (3\cdot 2^{i-1} - 2) \epsilon$ for $i\ge 1$.
\end{proof}

\begin{theorem}
\label{merge}
\textbf{(Merging Backprop-Q networks)} Two Backprop-Q networks can be merged at stochastic node $X$ and its ancestors, if the two are fully matched from $X$ through $X$'s ancestors, that is, the set of the incoming edges to each ancestor in a Backprop-Q network is exactly matched to the other.
\end{theorem}
\begin{proof}
This theorem is not difficult to prove. In a Backprop-Q network, the incoming edges to a Q-function node define its update rule. For a stochastic node, e.g., $Y$, if there are two Backprop-Q networks with respect to cost $f_1$ and $f_2$ respectively, both of which go through $Y$, node $Y$ will hold two Q-functions, denoted as $Q_{\scriptscriptstyle Y}^{f_1}$ and $Q_{\scriptscriptstyle Y}^{f_2}$. If the incoming edges to $Q_{\scriptscriptstyle Y}^{f_1}$ and those to $Q_{\scriptscriptstyle Y}^{f_2}$ are exactly the same, then the update rule of $Q_{\scriptscriptstyle Y}^{f_1}$, whether to take the sum or the average, would follow the same pattern as that of $Q_{\scriptscriptstyle Y}^{f_2}$. Therefore, we can treat the two Q-functions as one by summing them, denoted as $Q_{\scriptscriptstyle Y}^{f_1,f_2}$. We wish to propagate the new $Q_{\scriptscriptstyle Y}^{f_1,f_2}$ instead of the two functions $Q_{\scriptscriptstyle Y}^{f_1}$ and $Q_{\scriptscriptstyle Y}^{f_2}$, so that we need to make sure that all $Y$'s ancestor nodes could combine Q-functions in the same way. If the two Backprop-Q networks are fully matched over these ancestors, the merging for the downstream parts of the two networks, starting at $Y$, can occur through all $Y$'s ancestors. Actually, we can start the merging a little bit earlier, from the stochastic node $X$ one step after $Y$ in the SCG.
\end{proof}
\section{Gradient Difference Between Two Locally Fitted Functions}

When using a neural-network-based function approximator $Q_w(x)$ to fit $Q(x)$, we wish to know to what degree $Q_w(x)$ also preserves the first-order derivative information of $Q(x)$. If we could bound the difference of the gradients w.r.t. input $x$ between them, such that $\| \partial{Q_w}/\partial{x} - \partial{Q}/\partial{x} \| \le \epsilon$ for all $x$, and $\epsilon \rightarrow 0$ when $Q_w(x) \rightarrow Q(x)$ for all $x$, we can utilize the gradient of $Q_w(x)$ as well as its function value, and treat $ \partial{Q_w}/\partial{x}$ the same as $\partial{Q}/\partial{x}$. If we consider using the reparameterization trick for continuous or relaxed discrete random variables in some cases, the approximate gradient $\partial{Q_w}/\partial{x}$ can provide useful information even if the true $\partial{Q}/\partial{x}$ is unknown. However, this is not true universally. For example, a zigzag line, which is a piecewise linear function, can be infinitely close to a straight line but still keep its slope staying constant almost everywhere. Therefore, we need to impose some conditions to make the bounded gradient difference converge to zero. First, we propose a reasonable hypothesis on the function behavior of a neural network around a point.
\begin{hypothesis}
The functionality of a neural network $f(x)$ in a local region $x\in \Omega$ can be expressed fully by a family of polynomials $P$ of a finite degree $n$. \end{hypothesis}
This hypothesis assumes that the degree of the non-linearity of a neural network can be bounded locally. Then, we introduce the Bernstein's inequality \cite{Gardner1993_BernsteinInequalities}, which works as the theory basis to bound the gradient difference.
\begin{theorem}
\textbf{(Bernstein's inequality)} Let $P$ be a polynomial of degree $n$ with derivative $P'$. Then,
\begin{equation}
\max_{|z|\le 1} (|P'(z)|) \le n \cdot \max_{|z|\le 1}(|P(z)|)
\end{equation}
\end{theorem}
The above polynomial is defined on a scalar variable. The multivariate version of Bernstein's inequality can be found in \cite{Ditzian1992_MultivariateBernstein}. It shows that the magnitude (defined by $L_p$ norm) of the first-order derivative of a polynomial of degree $n$, in a bounded convex region such as $|z|\le 1$, can be bounded by the product of a constant relying on $n$ and the magnitude of the polynomial's value.

For simplicity, we consider a univariate case. Let $f(x) := Q_w(x) - Q(x)$ and we wish to bound $|f'(x) |= | Q_w'(x) - Q'(x) |$ by $|f(x)|=|Q_w(x) - Q(x)|$ for all $x \in [x_0 - \Delta x, x_0 + \Delta x]$. According to the above hypothesis, we can express $f(x)$ in a form of a polynomial of a finite degree $n$: $P(z) := f(x)$ where $z=\frac{x-x_0}{\Delta x}$ and $|z|\le 1$, and then have:
\begin{equation}
|f'(x)| = \Big|\frac{1}{\Delta x}P'(z)|_{z=\frac{x-x_0}{\Delta x}}\Big| \le \frac{n}{\Delta x} \cdot \max_{|z|\le 1}(|P(z)|) = \frac{n}{\Delta x} \cdot \max_{x\in [x_0 - \nabla x, x_0 + \nabla x]}(|f(x)|)
\end{equation}
We view $n/\Delta x$ as a constant $C$. Therefore, if we fit $Q_w(x)$ to $Q(x)$ well enough in a local region, that is $|Q_w(x) - Q(x)| \le \epsilon$ for all $x\in \Omega$, we can bound their gradient difference by $C\epsilon$ everywhere within this local region, converging to zero when $\epsilon \rightarrow 0$.
\section{Reduced Backprop-Q networks for fully-connected-layered SCGs}
\label{sec:reducebpq}
\begin{figure*}
\centering
\includegraphics[width=\textwidth]{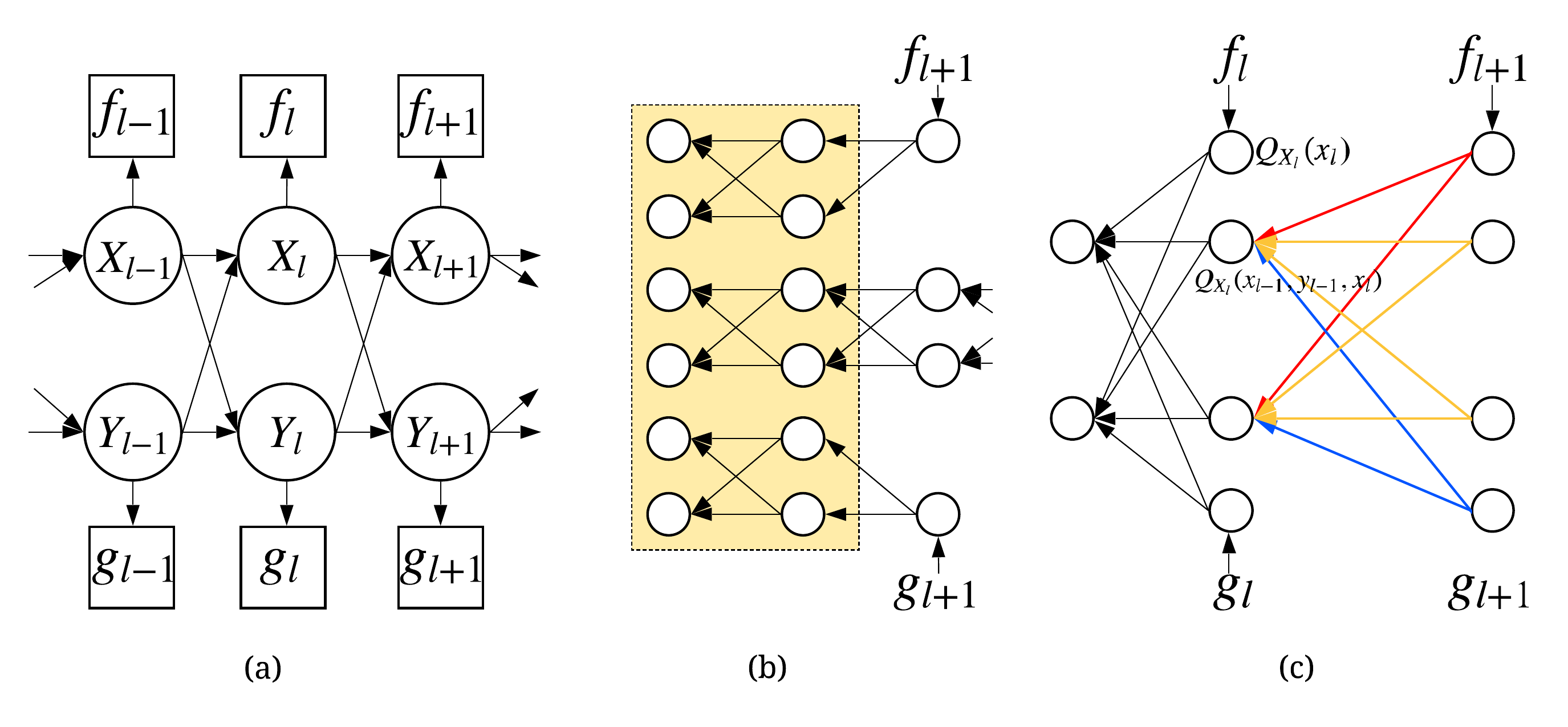}
\vspace{-20pt}
\caption{Reducing Backprop-Q networks for full-connected-layered SCGs}
\vspace{-10pt}
\label{fig-mc-c}
\end{figure*}

Figure \ref{fig-mc-c}(a) shows a multi-layer stochastic system, with each random variable relying on the entire previous stochastic layer, e.g., $X_l \sim p_{{\scriptscriptstyle X}_l}(\cdot|X_{l-1},Y_{l-1})$, and also associated with a cost $f_l(X_l)$. Suppose there are $N$ layers and $M$ nodes per layer, so that $N\cdot M$ costs will provide $N\cdot M$ separate Backprop-Q networks, and a node in layer $t$ needs to maintain $(N-t) M + 1$ Q-functions. However, from Figure \ref{fig-mc-c}(b), we can see that all the Backprop-Q networks rooted in layer $l+1$ and higher layers share the exactly same downstream subgraphs from layer $l$. This means, at each node in layer $l$, we can combine all the Q-functions provided by upstream Backprop-Q networks, into one Q-function like $Q_{{\scriptscriptstyle X}_l}(x_{l-1},y_{l-1},x_l)$. It takes red, yellow and blue incoming edges as shown in Figure \ref{fig-mc-c}(c), representing three different cost sources. Therefore, for each stochastic node $X$, we only need to maintain two Q-functions, one for the immediate cost defined on itself, one for the combined Q-functions from upstream. Further, we do not have to construct a new Q-function for the immediate cost but use it directly. As a result, each stochastic node $X$ only stores one Q-function that is $Q_{\scriptscriptstyle X_l}(x_{l-1},y_{l-1},x_l)$.

\section{Using Techniques from RL for Backprop-Q}

\subsection{Cases for $\lambda$-return Updates}

\begin{figure*}
\centering
\includegraphics[width=\textwidth]{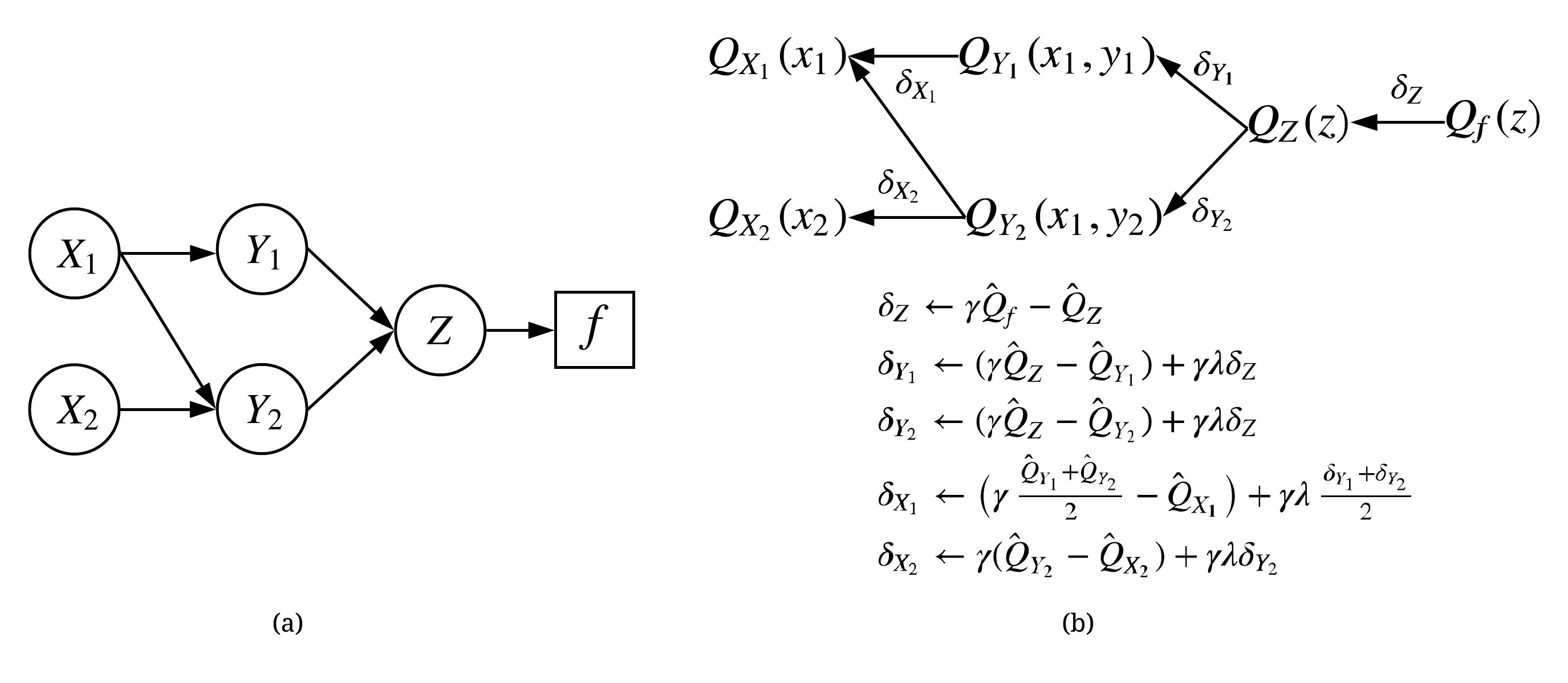}
\vspace{-20pt}
\caption{$\lambda$-return updates on an SCG with one cost}
\vspace{-10pt}
\label{fig-lambda-ret-a}
\end{figure*}

\begin{figure*}
\centering
\includegraphics[width=\textwidth]{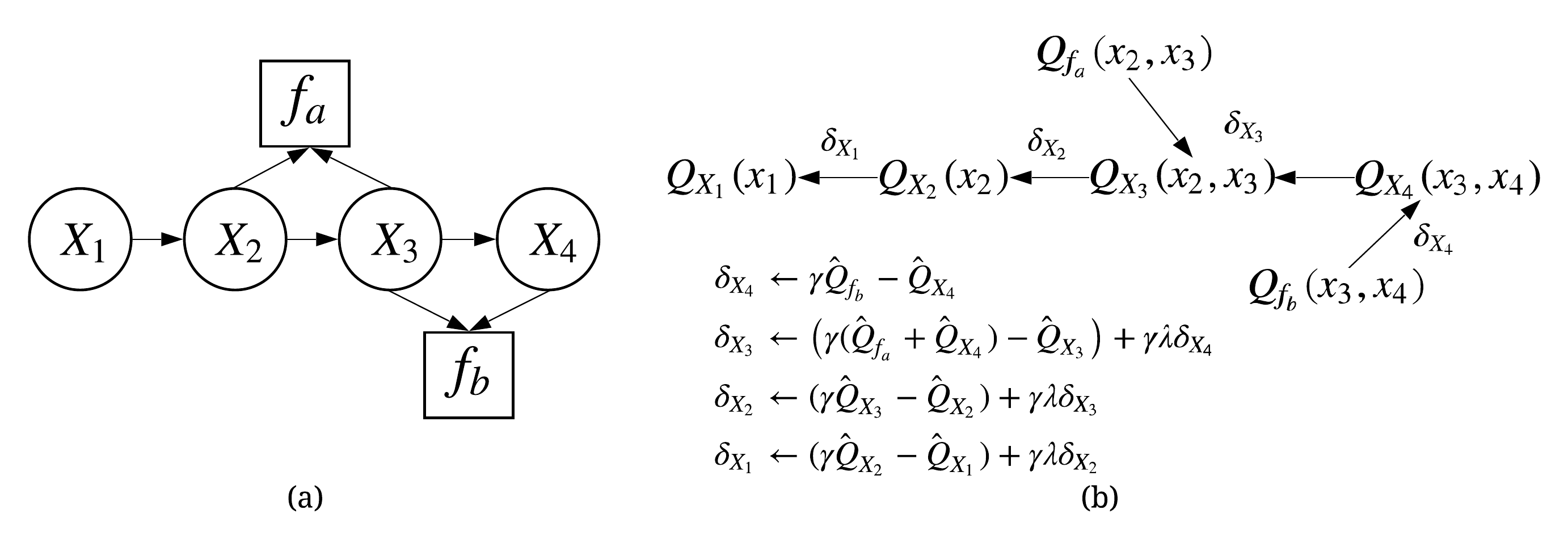}
\vspace{-20pt}
\caption{$\lambda$-return updates on an SCG with two costs}
\vspace{-10pt}
\label{fig-lambda-ret-b}
\end{figure*}

We show two cases to illustrate our graph-based $\lambda$-return method. Figure \ref{fig-lambda-ret-a}(a) is an SCG with one cost. Its corresponding Backprop-Q network is given in Figure \ref{fig-lambda-ret-a}(b), with propagated errors pointing to each node. The start error $\delta_{\scriptscriptstyle Z}$ is computed by $\gamma Q_f(z) - Q_{\scriptscriptstyle Z}(z)$ based on a sample $z$, where $\gamma$ is the discount factor. For simplicity of notation, we use $\hat{Q}_f:=Q_f(z)$ and $\hat{Q}_{\scriptscriptstyle Z}:=Q_{\scriptscriptstyle Z}(z)$. Then, each of the following errors is a sum of the current TD error and its upstream errors, like $\delta_{{\scriptscriptstyle Y}_1} \leftarrow (\gamma \hat{Q}_{\scriptscriptstyle Z} - \hat{Q}_{{\scriptscriptstyle Y}_1}) + \gamma \lambda \delta_{\scriptscriptstyle Z}$, where the second term is weighted by $\gamma \lambda$. For node $Q_{{\scriptscriptstyle X}_1}$, it has two upstream nodes $Q_{{\scriptscriptstyle Y}_1}$ and $Q_{{\scriptscriptstyle Y}_2}$ belonging to the same cost source. Therefore, we compute its TD error based on the averaged update target $(\hat{Q}_{{\scriptscriptstyle Y}_1} + \hat{Q}_{{\scriptscriptstyle Y}_2})/2$, and also average the two upstream errors $(\delta_{{\scriptscriptstyle Y}_1} + \delta_{{\scriptscriptstyle Y}_2})/2$.

The second case is a two-cost SCG in Figure \ref{fig-lambda-ret-b}(a). Before applying $\lambda$-return, we reduce its Backprop-Q networks into a simpler one, by removing edges $Q_{f_a} \to Q^{f_a}_{{\scriptscriptstyle X}_2}$ and $Q_{f_b} \to Q^{f_b}_{{\scriptscriptstyle X}_3}$ and merging the rest, as shown in Figure \ref{fig-lambda-ret-b}(b). The procedure is much like the first case, except that at node $Q_{{\scriptscriptstyle X}_3}$ we sum the two upstream Q-function values instead of averaging them due to different cost sources.

Taking $\delta_{{\scriptscriptstyle X}_1}$ as an example, in the first case, $\delta_{{\scriptscriptstyle X}_1} = \gamma^3 \hat{Q}_f - \hat{Q}_{{\scriptscriptstyle X}_1}$ if $\lambda = 1$, and $\delta_{{\scriptscriptstyle X}_1} = \gamma(\hat{Q}_{{\scriptscriptstyle Y}_1} + \hat{Q}_{{\scriptscriptstyle Y}_2})/2 - \hat{Q}_{{\scriptscriptstyle X}_1}$ if $\lambda = 0$; in the second case, $\delta_{{\scriptscriptstyle X}_1} = \gamma^3 \hat{Q}_{f_a} + \gamma^4 \hat{Q}_{f_b} - \hat{Q}_{{\scriptscriptstyle X}_1}$ if $\lambda = 1$, and $\delta_{{\scriptscriptstyle X}_1} = \gamma\hat{Q}_{{\scriptscriptstyle X}_2} - \hat{Q}_{{\scriptscriptstyle X}_1}$ if $\lambda = 0$. This gives us a more flexible way to make a compromise between bias and variance. When being at an early phase of training, we set $\lambda$ and $\gamma$ close to 1, so that the remote cost signal can propagate backward faster; after training Q-functions for a while, we decrease $\lambda$ a little bit to reduce variance by relying more on learned Q-functions and thus cumulative effect of past experience.

\subsection{Cases for Experience Replay}

\begin{figure*}
\centering
\includegraphics[width=\textwidth]{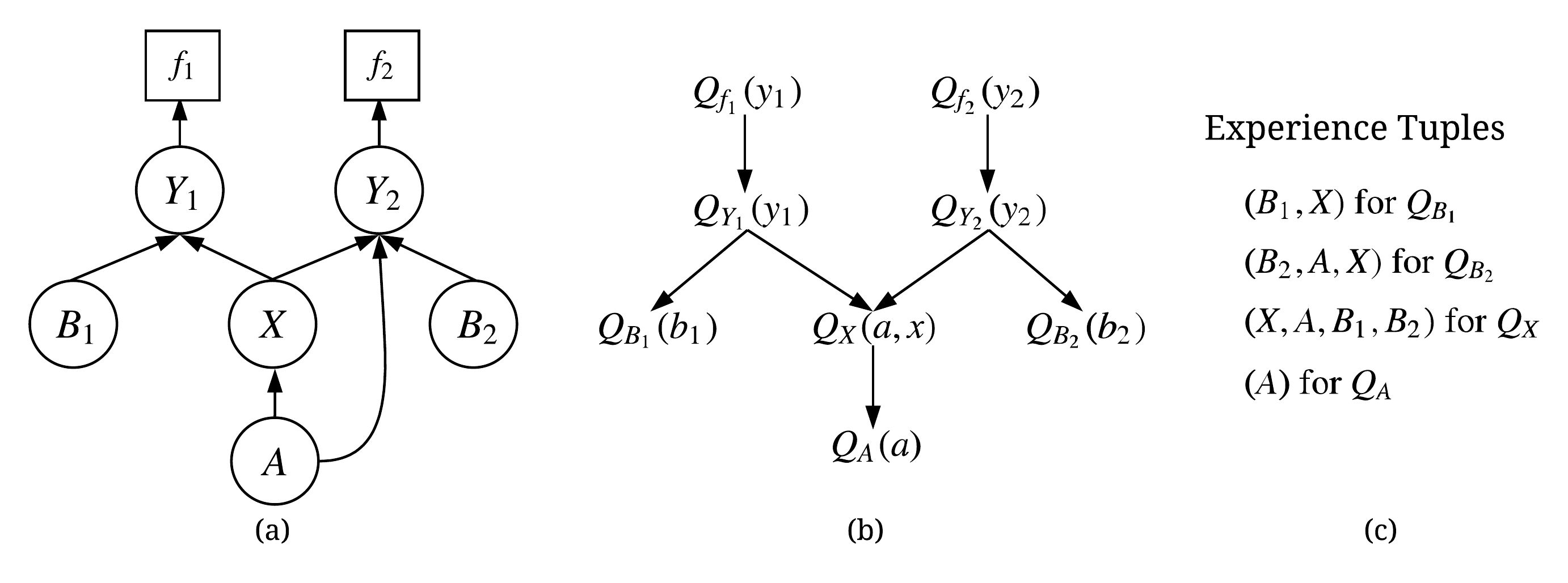}
\vspace{-20pt}
\caption{Graph-based experience replay in an SCG}
\vspace{-10pt}
\label{fig-er}
\end{figure*}

For the purpose of illustration, we consider an SCG with two costs shown in Figure \ref{fig-er}(a). For simplicity, we remove edge $Q_{{\scriptscriptstyle Y}_2}^{f_2} \to Q_{\scriptscriptstyle A}^{f_2}$ to get a simplified Backprop-Q network in Figure \ref{fig-er}(b). We generate and store four types of experience tuples for nodes $Q_{{\scriptscriptstyle B}_1},Q_{{\scriptscriptstyle B}_2},Q_{\scriptscriptstyle X}$ and $Q_{\scriptscriptstyle A}$ respectively, as shown in Figure \ref{fig-er}(c). Taking $Q_{\scriptscriptstyle X}$ as an example, its experience tuple should contain $X$ and $A$ for the scope of $Q_{\scriptscriptstyle X}$, and also include $B_1$ and $B_2$ to generate $Y_1$ and $Y_2$ respectively, together with $X$. Given an experience sample $(x,a,b_1,b_2)$, a sample $y_1$ should be drawn from $p_{{\scriptscriptstyle Y}_1}(\cdot|x,b_1;\theta_{{\scriptscriptstyle Y}_1})$, and $y_2$ drawn from $p_{{\scriptscriptstyle Y}_2}(\cdot|x,b_2;\theta_{{\scriptscriptstyle Y}_2})$, both based on the current policy parameters. This sampling process can be performed multiple times to generate many $(y_1,y_2)$ for training $Q_{w_X}$ by taking gradient steps to minimize:
\begin{equation}
L(w_{\scriptscriptstyle X}) = \frac{1}{n} \sum_{i=1}^n \Big( Q_{w_{Y_1}}(y_1^{(i)}) + Q_{w_{Y_2}}(y_2^{(i)}) - Q_{w_X}(a,x) \Big)^2
\end{equation}

\subsection{Details for Slow-tracking Target.} DQN \cite{Mnih2015_HumanLevel} uses a separate network for generating the targets for the Q-learning updates, which takes parameters from some previous iteration and updates periodically. DDPG \cite{Lillicrap2016_DDPG} uses "soft" target updates rather than directly coping the weights. Our solution, called slow-tracking target, is similar to target updates in DDPG by having the target network slowly track the learned network. This can be applied to both policy parameters and critic parameters when performing experience replay. For each parameter, we maintain $\theta_t$ and $\Delta \theta_t$, where $\theta_t$ represents the parameter of the target network, and $\theta_t + \Delta \theta_t$ represents the parameter of the current learned network. We suppose that the varying of $\theta_t$ is slow while $\Delta \theta_t$ can change drastically. Each time we obtain a new $\Delta$, we add it to $\Delta \theta_t$ as: $\Delta \theta_{t+1} \leftarrow \Delta \theta_t + \Delta$, and then we let $\theta_t$ slowly track the new $\Delta \theta_{t+1}$ as: $\theta_{t+1} \leftarrow \theta_t + \alpha \Delta \theta_{t+1}$ and $\Delta \theta_{t+1} \leftarrow (1-\alpha)\Delta \theta_{t+1}$ with a positive $\alpha \ll 1$.

\begin{figure*}
\centering
\includegraphics[width=\textwidth]{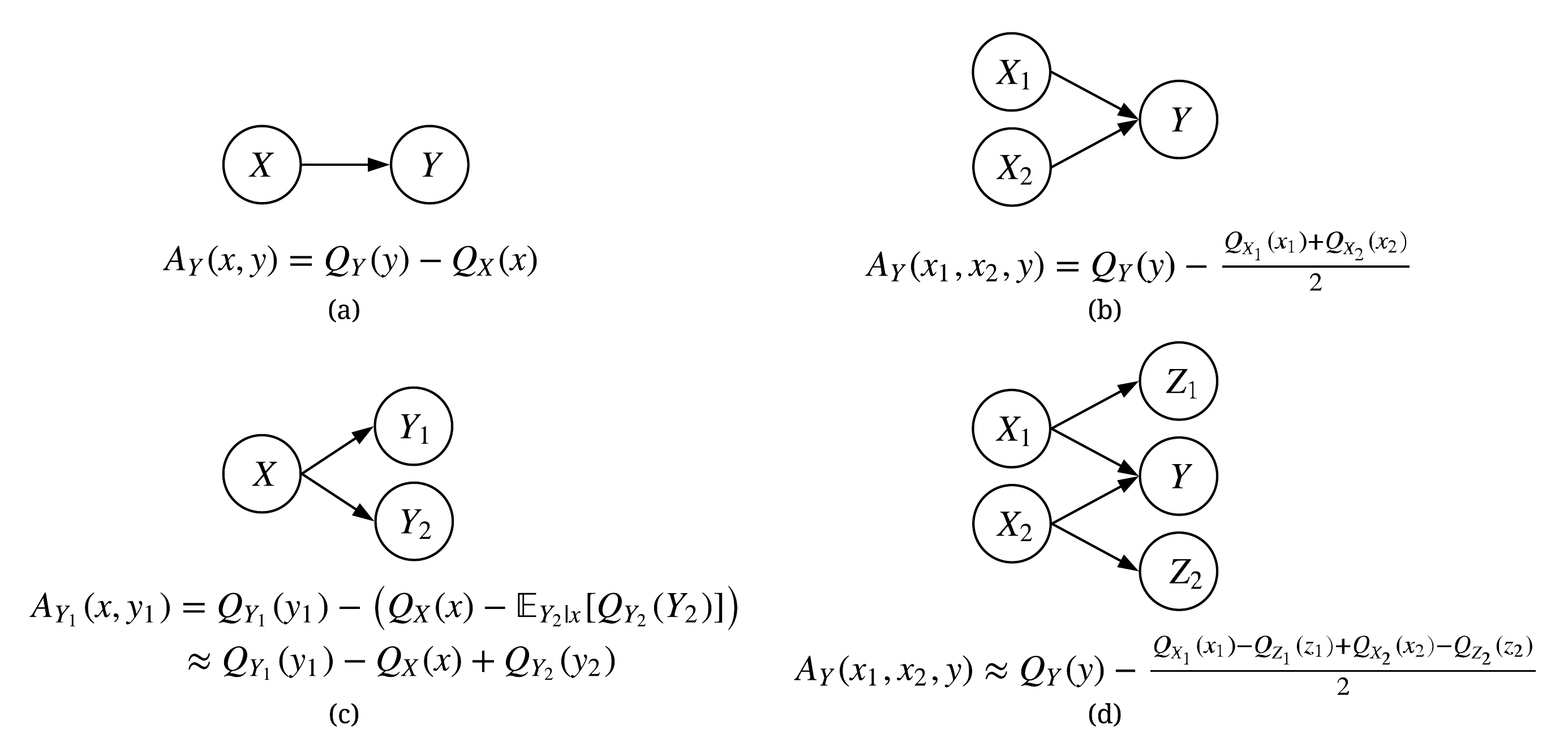}
\vspace{-20pt}
\caption{Computing advantage functions in SCGs}
\label{fig-adv-a}
\end{figure*}

\subsection{Cases for Advantage Functions.} In reinforcement learning, the action-value function summarizes the performance of each action from a given state, assuming it follows $\pi$ thereafter, while the advantage function provides a measure of how each action compares to the average performance at state $s_t$ given by the state-value function. The advantage function is often used to replace the action-value function in the gradient estimator to yield much lower variance \cite{Mnih2016_A3C,Schulman2016_GAE,Gu2017_QProp}. It is viewed as an instance of the baseline method, centering the learning signal and reducing variance significantly. To derive the SCG version of advantage functions, we take Figure \ref{fig-adv-a}(a) as an example. Here, suppose that $Y$ is action and $X$ is state, so that $Q_{\scriptscriptstyle Y}(y)$ represents the performance for taking action $y$, and $Q_{\scriptscriptstyle X}(x) = \mathbb{E}_{{\scriptscriptstyle Y}|x}[Q_{\scriptscriptstyle Y}(Y)]$ represents the average performance for taking all actions at state $x$. Therefore, the advantage function at $Y$ should be $A_{\scriptscriptstyle Y}(x,y) = Q_{\scriptscriptstyle Y}(y) - Q_{\scriptscriptstyle X}(x)$. If $Y$ has two states as in Figure \ref{fig-adv-a}(b), each of $Q_{{\scriptscriptstyle X}_1}(x_1)$ and $Q_{{\scriptscriptstyle X}_2}(x_2)$ gives an evaluation of the average performance at its own state. We thus subtract the two's average from $Q_{\scriptscriptstyle Y}(y)$ to compute the advantage function $A_{\scriptscriptstyle Y}(x_1,x_2,y)$ at $Y$. In Figure \ref{fig-adv-a}(c)(d), the advantage functions become more complex, requiring us to consider other branches. For example, in Figure \ref{fig-adv-a}(c), as $Q_{\scriptscriptstyle X}(x)$ takes the sum of $Q_{{\scriptscriptstyle Y}_1}(y_1)$ and $Q_{{\scriptscriptstyle Y}_2}(y_2)$ as its update target, when computing $A_{{\scriptscriptstyle Y}_1}(x,y_1)$ at $Y_1$, we need to subtract $\mathbb{E}_{{\scriptscriptstyle Y}_2|x}[Q_{{\scriptscriptstyle Y}_2}(Y_2)]$ from $Q_{\scriptscriptstyle X}(x)$. Here, we approximate the advantage function $A_{{\scriptscriptstyle Y}_1}(x,y_1)$ by using a sample $Q_{{\scriptscriptstyle Y}_2}(y_2)$ instead of the expectation computation. In practice, the above advantage functions are not known and must be estimated as we estimate the Q-functions by $Q_w$. One way is to build the approximate advantage functions directly based on $Q_w$, such as $A_{w_Y}(x,y) := Q_{w_Y}(y) - Q_{w_X}(x)$ for the case in Figure \ref{fig-adv-a}(a). Another way is to use $\lambda$-return to estimate the first term so that we can utilize the remote signal in case that $Q_w$ is not accurate yet. Figure \ref{fig-adv-b} shows that the advantage function at $X_t$ can be approximated by the error $\delta_{t-1}$. In the extreme case when $\gamma \lambda = 1$, $A_{{\scriptscriptstyle X}_t}(x_{t-1},x_t)$ reduces to $R - Q_{w_{X_{t-1}}}(x_{t-1})$ where $R$ represents the actual return.

\begin{figure*}
\centering
\includegraphics[width=0.6\textwidth]{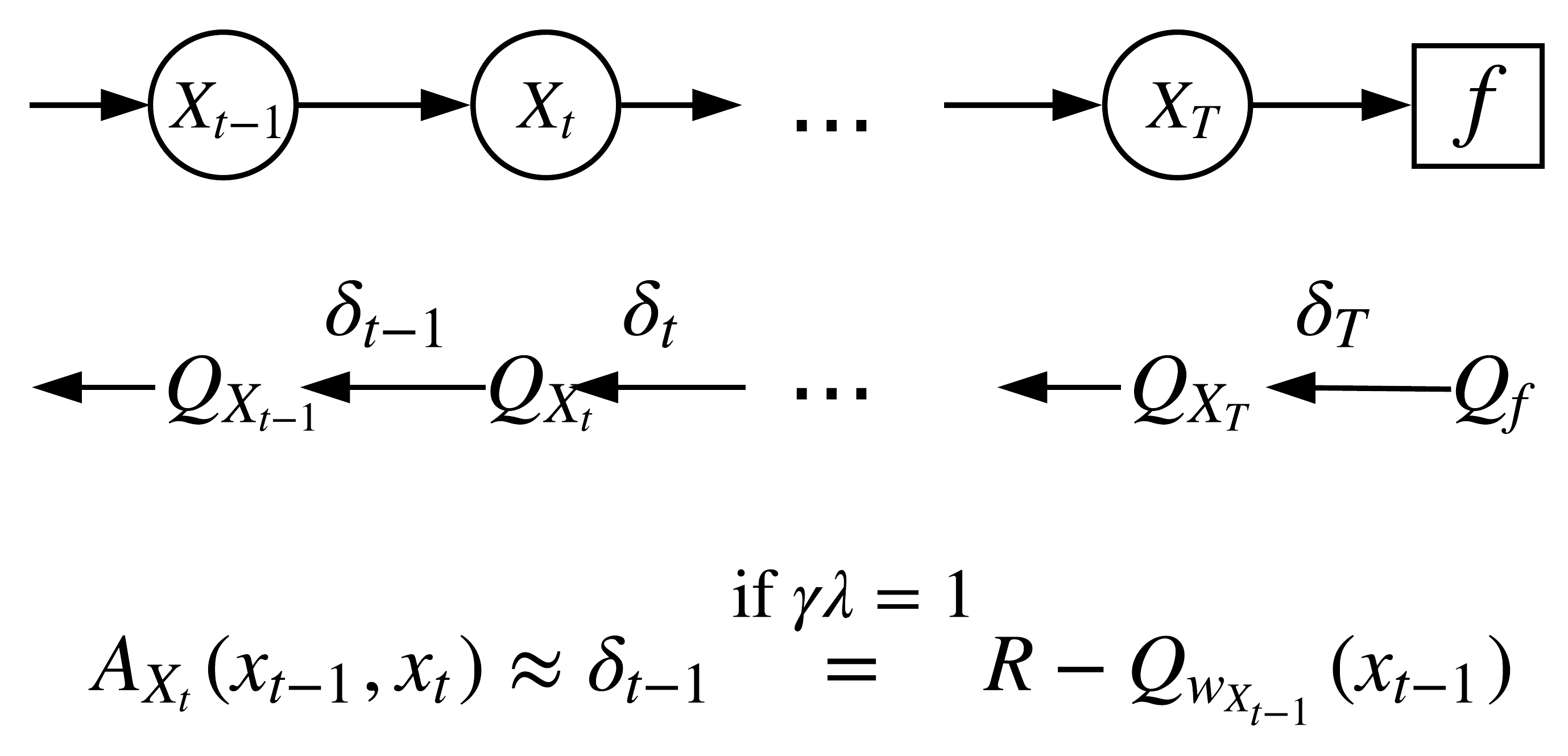}
\vspace{-5pt}
\caption{Approximating advantage functions using $\lambda$-return errors}
\vspace{-10pt}
\label{fig-adv-b}
\end{figure*}

\subsection{Details for Controlled Policy Optimization.}

To avoid an excessively large policy update, TRPO \cite{Schulman2015_TRPO} puts a constraint on the change in the policy at each update, and guarantees policy monotonic improvement with controlled step sizes. It solves a constrained optimization problem on the policy parameters. In the context of SCGs, consider a pair $(X,{Pa}_{\scriptscriptstyle X})$ where $X\sim p(\cdot|{Pa}_{\scriptscriptstyle X}; \theta)$. The constrained optimization problem is:
\begin{equation}
\begin{split}
\min_\theta &\quad \mathbb{E}_{{\scriptscriptstyle X}\sim p(\cdot|{Pa}_{X; \theta_\text{old}})} \bigg[ \frac{p(X|{Pa}_{\scriptscriptstyle X}; \theta)}{p(X|{Pa}_{\scriptscriptstyle X}; \theta_\text{old})} Q_{\scriptscriptstyle X}({Sc}_{\scriptscriptstyle X})\bigg] \\
\text{s.t.} &\quad D_{KL}\Big(p(\cdot|{Pa}_{\scriptscriptstyle X}; \theta) \| p(\cdot|{Pa}_{\scriptscriptstyle X}; \theta_\text{old})\Big) \le \delta
\end{split}
\end{equation}
which can be solved by the conjugate gradient algorithm followed by a line search. The objective function is approximated by drawing multiple samples, while the constraint is approximated by a quadratic approximation using the Fisher information matrix.

PPO \cite{Schulman2017_PPO} introduces a much simpler way to implement the controlled policy optimization. In contrast to TRPO, it uses a clipped surrogate objective without any constraint. Given a pair $(X,{Pa}_{\scriptscriptstyle X})$ in a SCG, we write the objective as:
\begin{equation}
\min_\theta \mathbb{E}_{\scriptscriptstyle X} \Big[ \max \Big( r_{\scriptscriptstyle X}(\theta) Q_{\scriptscriptstyle X}({Sc}_{\scriptscriptstyle X}), \;  \text{clip}\big(r_{\scriptscriptstyle X}(\theta), 1-\epsilon, 1+\epsilon\big) Q_{\scriptscriptstyle X}({Sc}_{\scriptscriptstyle X}) \Big) \Big]
\end{equation}
where $r_{\scriptscriptstyle X}(\theta) = \frac{p(X|{Pa}_{\scriptscriptstyle X};\theta)}{p(X|{Pa}_{\scriptscriptstyle X};\theta_\text{old})}$. The idea behind this is to remove the incentive for moving $r_{\scriptscriptstyle X}(\theta)$ outside of the interval $[1-\epsilon, 1+\epsilon]$.
\section{Using Techniques from VB for Backprop-Q}

\subsection{Graphical Notation}

\begin{figure*}
\centering
\includegraphics[width=\textwidth]{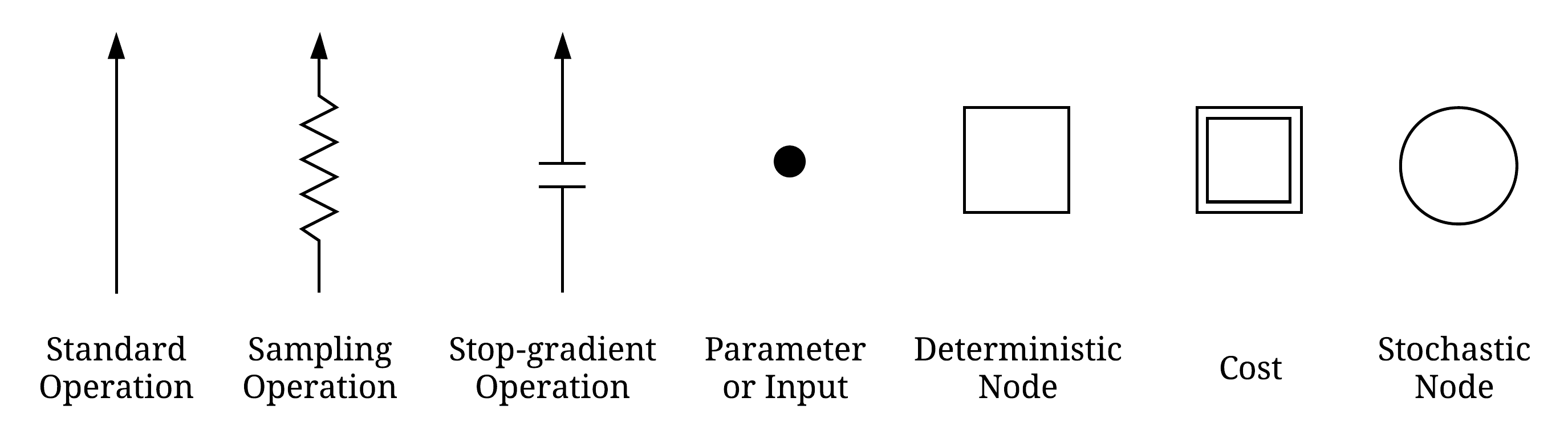}
\vspace{-5pt}
\caption{Graphical Notation}
\label{fig-notation}
\end{figure*}

Here, we list some gradient estimation methods, illustrated in SCGs. To explain it well, we use a new graphical notation that differentiates three types of arrows as shown in Figure \ref{fig-notation}. The arrow, called \emph{standard operation}, represents a normal deterministic computation, producing an output when given an input. The arrow, called \emph{sampling operation}, represents a process of drawing a sample from a distribution. The arrow, called \emph{stop-gradient operation}, is also a deterministic mapping but with no flowing-back gradients permitted. Gradients can only be propagated backward through standard operations if not specified. For the notation of nodes, we use a double-line square to denote a cost node, and the rest follows \cite{Schulman2015_SCG}.

\subsection{REINFORCE / Score-Function / Likelihood-Ratio Estimators}

\begin{figure*}
\centering
\includegraphics[width=0.75\textwidth]{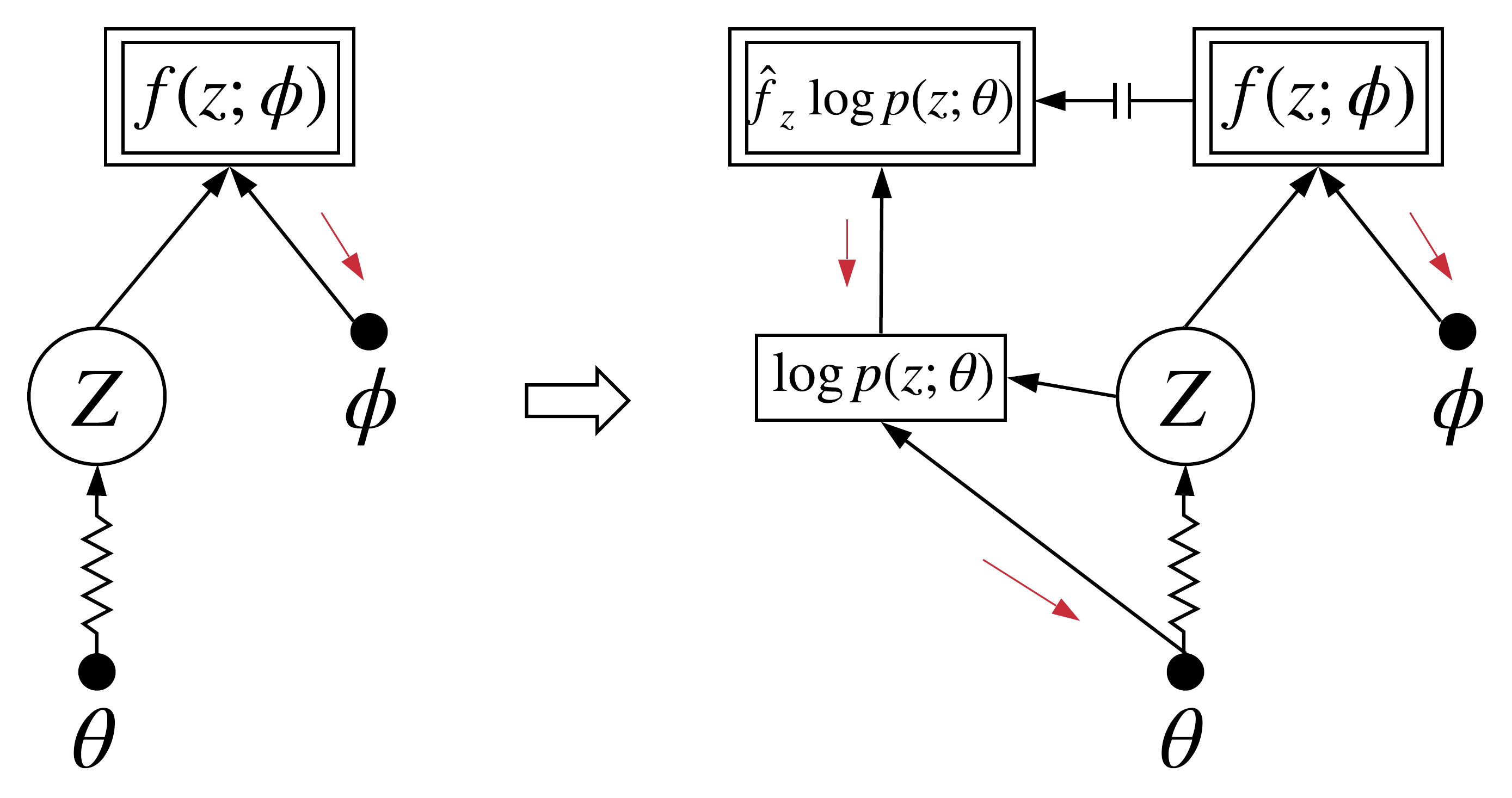}
\vspace{-5pt}
\caption{REINFOCE / Score-Function / Likelihood-Ratio Estimators}
\label{fig-reinforce}
\end{figure*}

The REINFORCE estimator \cite{Williams1992}, also known as the score-function or the likelihood-ratio estimator, has the widest applicability for almost all cases. It does not require $f(z)$ to be differentiable, and only needs a function output to estimate the gradient. Also, it is applicable to both continuous and discrete random variables. We illustrate the REINFORCE estimator using an SCG in Figure \ref{fig-reinforce}. The left graph shows an optimization problem containing a sampling operation, where only $\phi$ can receive a signal of the cost gradient. To send a gradient signal to $\theta$, we have to create a surrogate objective function, $\hat{f}_z \log p(z;\theta)$, taking an output of function $f(z)$, denoted as $\hat{f}_z$, with no gradients allowed to send back to $f(z)$. Then, we build a differentiable path from $\theta$ to the surrogate objective that can propagate the gradient back to $\theta$. Note that as a Monte-Carlo-based method, a forward pass for computing $\hat{f}_z \log p(z;\theta)$ involves a sampling operation to draw $z$ from $\theta$, which is the source of stochasticity and thus brings variance. The REINFORCE gradient estimator w.r.t $\theta$ is written as:
\begin{equation}
\hat{g} := f(z) \frac{\partial}{\partial \theta}\log p(z;\theta)
\end{equation}

\subsection{Control Variates}

\begin{figure*}
\centering
\includegraphics[width=\textwidth]{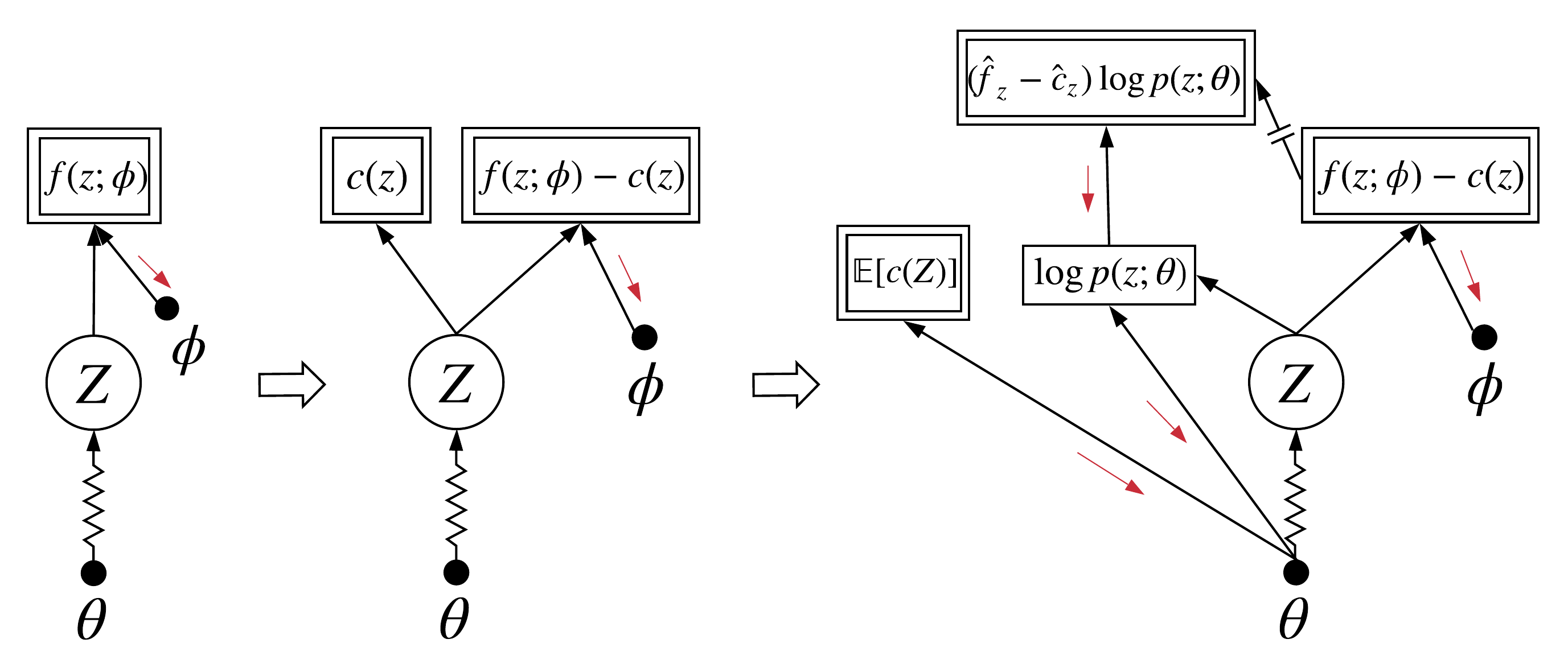}
\vspace{-5pt}
\caption{Control Variates}
\label{fig-cv}
\end{figure*}

Control variates is a variance reduction technique that helps reduce the high variance in the REINFORCE estimator. A carefully designed control variate should be a term highly correlated to $f(z)$, with a closed-form or tractable expectation to correct the bias. Figure \ref{fig-cv} shows how a control variate works to yield an unbiased and low-variance gradient estimator. First, we subtract a term $c(z)$ from $f(z)$ and add the same one aside simultaneously to keep the total cost unbiased. Then, we create a surrogate objective $(\hat{f}_z - \hat{c}_z) \log p(z;\theta)$ the same way as in the previous subsection, reachable from $\theta$ via a differentiable path. Since $\mathbb{E}_{\scriptscriptstyle Z}[c(Z)]$ can be computed analytically, we avoid the operation of sampling $z$ to connect $\theta$ to $\mathbb{E}_{\scriptscriptstyle Z}[c(Z)]$ directly, resulting in no variance when estimating the gradient of this part. To reduce variance, we wish $c(z)$ to be closely correlated to $f(z)$ so that the magnitude of $f(z) - c(z)$ could be as small as possible. There are several ways to design $c(z)$. (1) Let $c$ be a constant, a moving average, or a function that does not rely on $z$ \cite{Mnih2014_NVIL}. Due to $\frac{\partial}{\partial \theta} \mathbb{E}_{\scriptscriptstyle Z}[c(Z)]=0$, we can remove the edge from $\theta$ to $\mathbb{E}_{\scriptscriptstyle Z}[c(Z)]$. This method is often called \emph{baseline}. (2) Let $c(z)$ be the linear Taylor expansion of $f(z)$ around $z=\mathbb{E}_{\scriptscriptstyle Z}[Z]$ \cite{Gu2016_MuProp}:
\begin{equation}
c(z) = f(\mathbb{E}_{\scriptscriptstyle Z}[Z]) + f'(z)\Big|_{z=\mathbb{E}_{\scriptscriptstyle Z}[Z]}( z - \mathbb{E}_{\scriptscriptstyle Z}[Z])
\end{equation}
Then we have:
\begin{equation}
\frac{\partial}{\partial \theta} \mathbb{E}_{\scriptscriptstyle Z}[c(Z)]=f'(z)\Big|_{z=\mathbb{E}_{\scriptscriptstyle Z}[Z]} \frac{\partial}{\partial \theta} \mathbb{E}_{\scriptscriptstyle Z}[Z]
\end{equation}
where $f'(z)\big|_{z=\mathbb{E}_{\scriptscriptstyle Z}[Z]}$ is computed through a deterministic and differentiable mean-field network. Furthermore, to learn a good control variate, we minimize the expected square of the centered learning signal by $\min_c \mathbb{E}_{\scriptscriptstyle Z}[(f(Z)-c)^2]$, or maximize the variance reduction by learning the best scale factor $a$ in $\hat{f}(z) = f(z)-a \cdot c(z)$, so that $\text{Var}(\hat{f})$ has the minimal value when $a = {\text{Cov}(f,c)}/{\text{Var}(c)}$. The gradient estimator with a control variate is written as:
\begin{equation}
\hat{g} := (f(z) - c(z)) \frac{\partial}{\partial \theta}\log p(z;\theta) + \frac{\partial}{\partial \theta}\mathbb{E}_{\scriptscriptstyle Z}[c(Z)]
\end{equation}

\subsection{Reparameterization Trick}

\begin{figure*}
\centering
\includegraphics[width=0.7\textwidth]{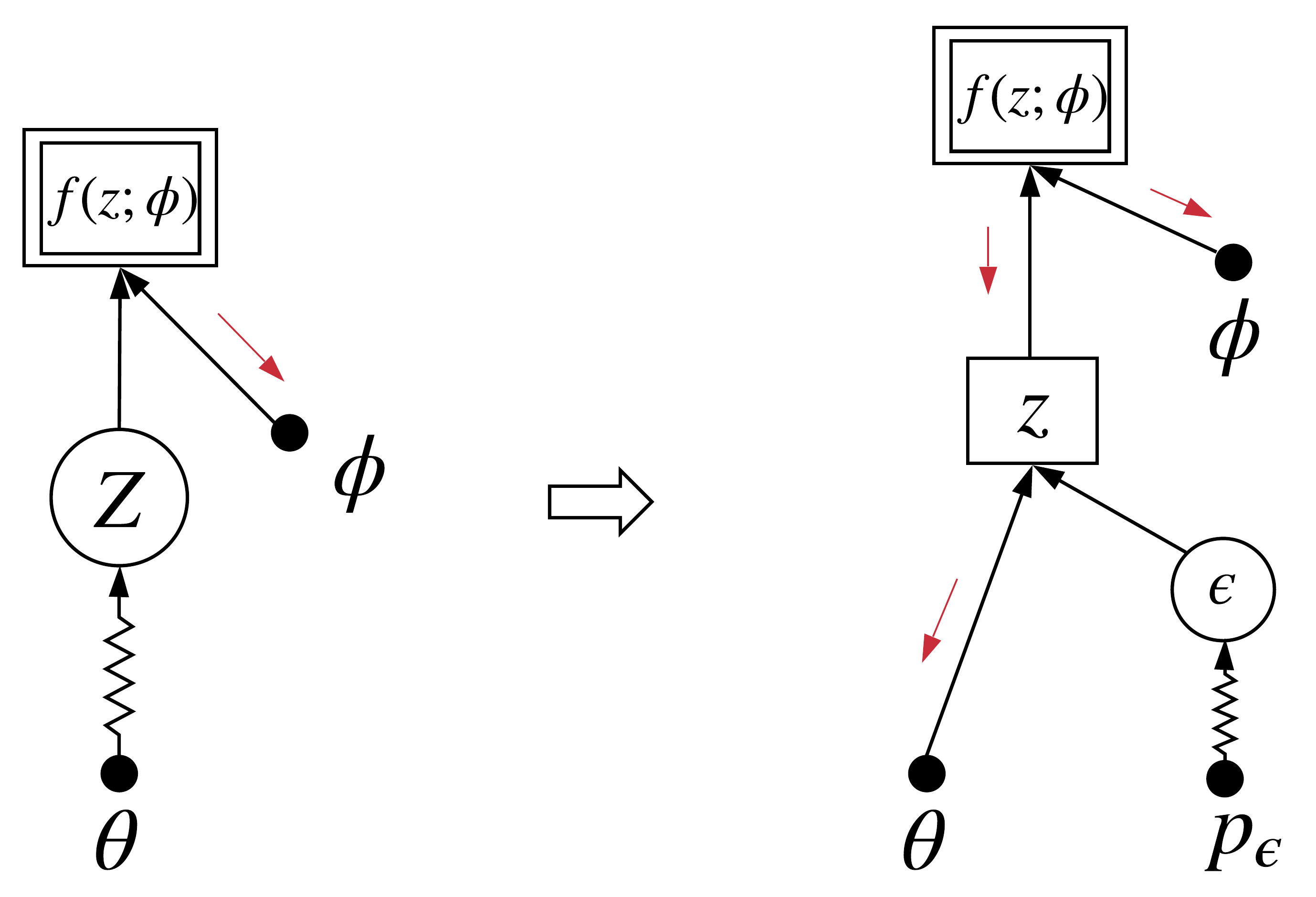}
\vspace{-5pt}
\caption{Reparameterization Trick}
\label{fig-rt}
\end{figure*}

The reparameterization trick is thought to yield an unbiased gradient estimator with lower variance than that of control variates in practice. The intuition is that it takes advantage of the first-order derivative $\partial f / \partial z$, while the control variates method only uses an outcome like $\hat{f}_z$. However, the reparameterization estimator requires $z$ to be continuous and $f(z)$ to be differentiable. Most importantly, we need to find a transformation function $z(\epsilon;\theta)$ where $\epsilon$ is a sample from a fixed known distribution and $\theta$ is the distribution parameter such that $z\sim p(\cdot;\theta)$ can follow exactly the same distribution as before. Therefore, it is typically used with Gaussian distribution \cite{Kingma2014_VAE,Rezende2014_DLGM}. \cite{Ruiz2016_GREP} proposed a generalized reparameterization gradient for a wider class of distributions, but it demands a sophisticated invertible transformation that it is not easy to define. The right graph in Figure \ref{fig-rt} shows a deterministic node of $z$ in place of the stochastic node $Z$, as $z$ is computed by a function of $\theta$ and $\epsilon$ but not sampled directly. Therefore, we can propagate the gradient signal through $z$ to $\theta$. The reparameterization gradient estimator is written as:
\begin{equation}
\hat{g} := \frac{\partial f}{\partial z} \frac{\partial}{\partial \theta}z(\epsilon;\theta)
\end{equation}

\subsection{Continuous Relaxation + Reparameterization Trick}

\begin{figure*}
\centering
\includegraphics[width=\textwidth]{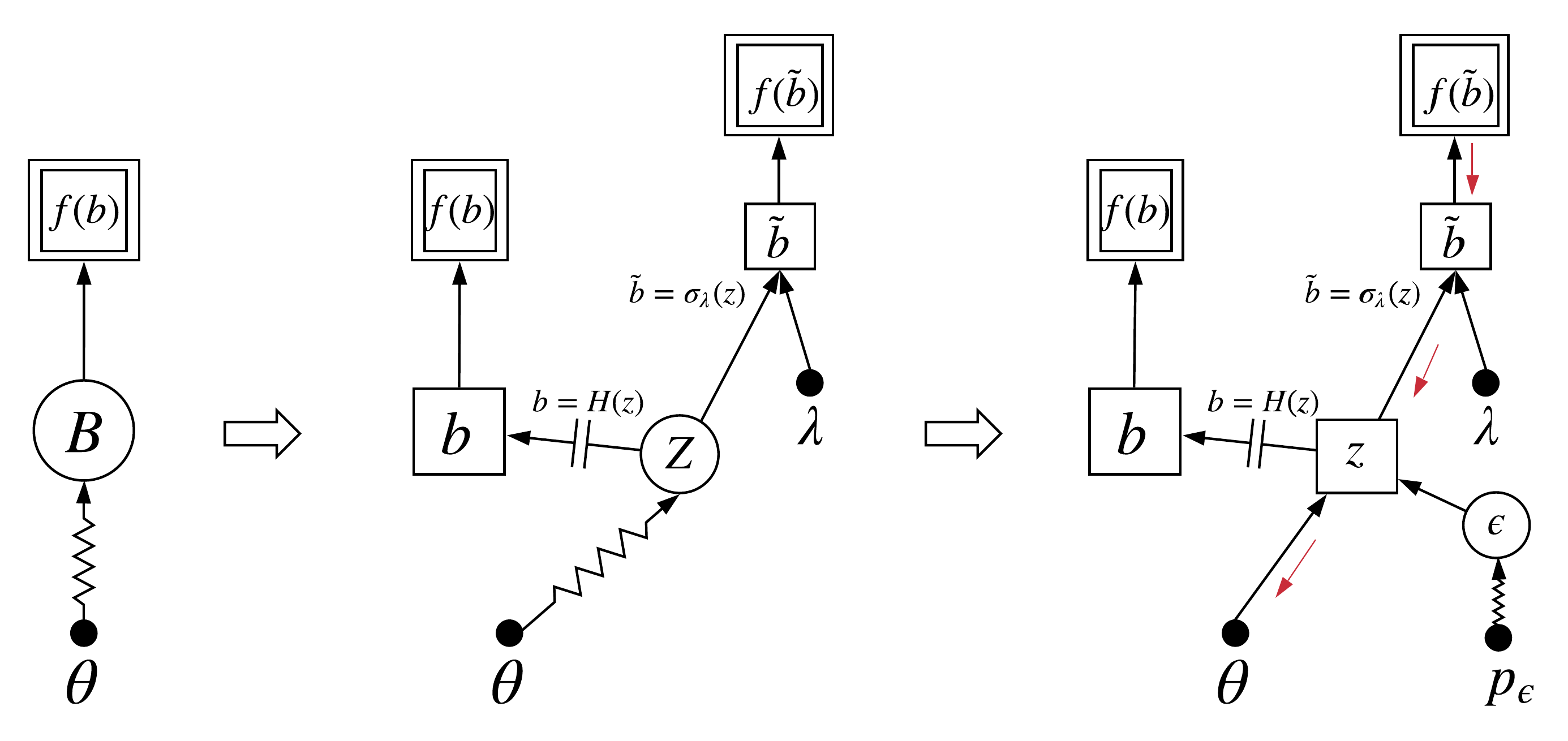}
\vspace{-5pt}
\caption{Combining continuous relaxation with reparameterization trick.}
\label{fig-relaxation}
\end{figure*}

How can we apply the reparameterization trick to discrete random variables, so that we can utilize the gradient information of a cost function to further reduce variance? \cite{Maddison2017_Concrete,Jang2017_GumbelSoftmax} introduced the Concrete distribution and the Gumbel-Softmax distribution respectively to build relaxed discrete models. A discrete random variable can be a binary variable $B\sim \text{Bernoulli}(\theta)$, or a categorical variable $B\sim \text{Categorical}(\theta)$ represented by an one-hot vector. Instead of sampling $b$ directly, we draw a continuous sample $z$ from the Gumbel distribution as shown in Figure \ref{fig-relaxation}. The Gumbel random variable $Z$ can be reparameterized by a transformation function of $\theta$ and a noise $\epsilon$ from the uniform distribution. Then $b$ can be computed through a hard threshold function, $b=H(z)$. However, the threshold function provides zero gradients almost everywhere, blocking any upstream gradient signal. To solve it, we introduce a sigmoid function $\sigma_\lambda(z)$ with a temperature hyperparametr $\lambda$ to produce a relaxed $\tilde{b}$. Instead of minimizing cost $f(b)$, we minimize $f(\tilde{b})$ and open up a differentiable path from $\theta$ to $f(\tilde{b})$, which absolutely brings in biases due to $\tilde{b}$ not being $b$. However, in the low temperature limit when $\lambda \to 0$, we have $\tilde{b} \to b$ and thus obtain an unbiased estimator. The gradient estimator is written as:
\begin{equation}
\hat{g} := \frac{\partial f}{\partial \tilde{b}}\Big|_{\tilde{b}= \sigma_\lambda(z)} \frac{\partial \sigma_\lambda}{\partial z} \frac{\partial}{\partial \theta}z(\epsilon;\theta)
\end{equation}

\subsection{Control Variates + Reparameterization Trick}

\begin{figure*}
\centering
\includegraphics[width=\textwidth]{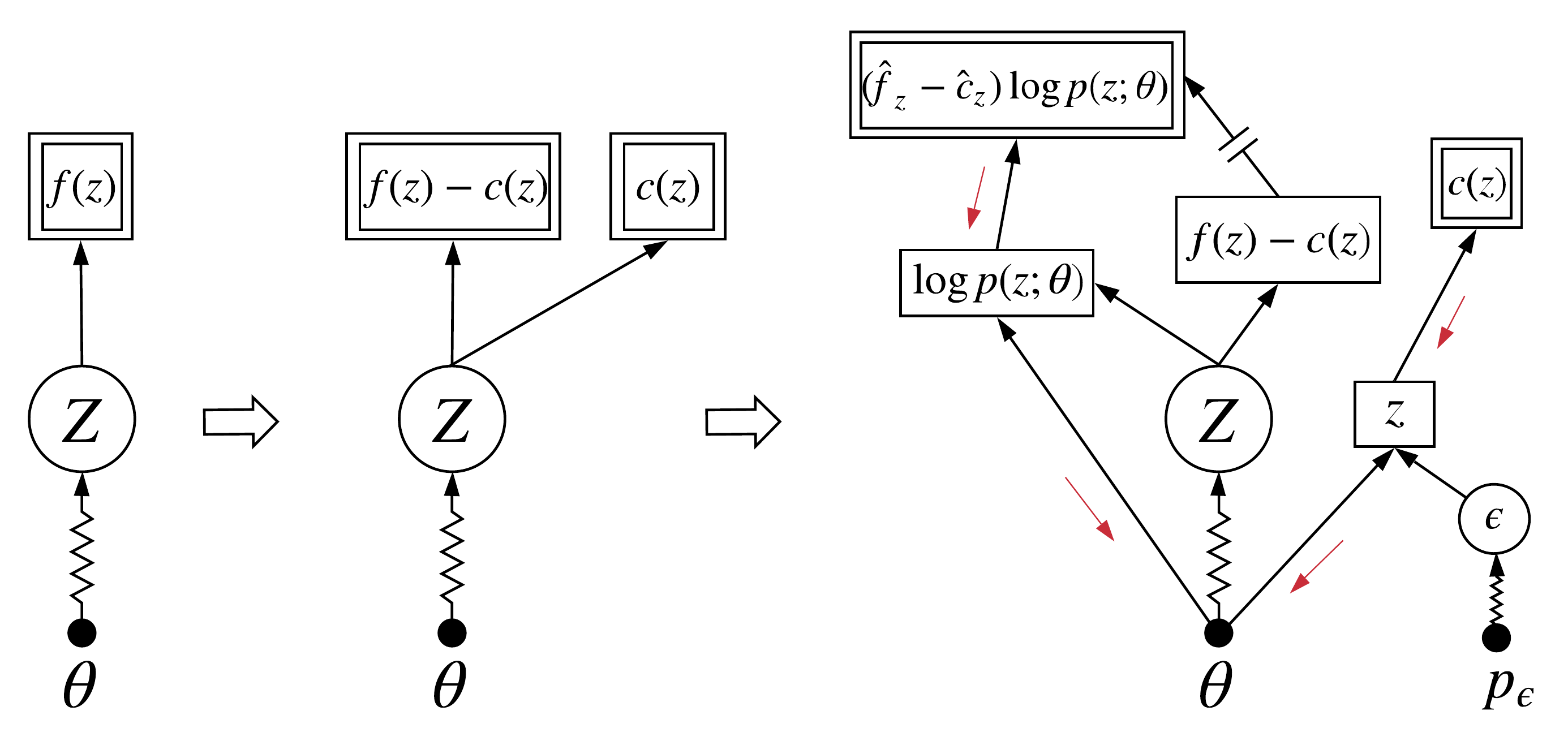}
\vspace{-5pt}
\caption{Combining control variates with reparameterization trick.}
\label{fig-cv-rt}
\end{figure*}

For the control variates method, how well it reduces the variance depends on how correlated to $f(z)$ the control variate is. The effective way is to design a sample-dependent control variate like $c(z)$ rather than a constant or a baseline, so that it can change the value as $z$ is changing, keeping closely correlated to $f(z)$. However, it may introduce bias, so we need  a known mean of $c(z)$ to correct the bias. Unfortunately, that limits the possible forms $c(z)$ can take. Inspired by the reparameterizaton trick, we take a compromise solution that a reparameterization gradient estimator is used in place of the gradient of the true mean. That only requires $c(z)$ to be differentiable and $z$ to be continuous. In practice, the reparameterization estimator usually yields lower variance compared to control variates. Therefore, we provide an unbiased and lower-variance gradient estimator by combining control variates with the repameterization trick \cite{Grathwohl2018_RELAX}.

In Figure \ref{fig-cv-rt}, we suppose $f(z)$ a non-differentiable or even unknown cost function, treated as a black-box function. We can acquire no more information about $f$ than a function output queried by an input $z$. We design a differentiable surrogate $c(z)$ to approximate $f(z)$ and apply it from two aspects: (1) Let $c(z)$ be a control variate, subtracted from $f(z)$ to reduce its variance. (2) Consider that $c(z)$ has its first-order derivative approximate well to that of $f(z)$, so that we can utilize the gradient information with the reparameterization trick, transporting the signal of ${\partial c}/{\partial z}$ from the bias-correction term $c(z)$ through $z$ to $\theta$. Thus, we build two paths from $\theta$ to costs via which the gradient signals can be sent back. The gradient estimator w.r.t. $\theta$ is written as:
\begin{equation}
\hat{g} := (f(z) - c(z))\frac{\partial}{\partial \theta}\log p(z;\theta) + \frac{\partial c}{\partial z} \frac{\partial}{\partial \theta}z(\epsilon;\theta)
\end{equation}
Generally, $c(z;w)$ is parameterized by a neural network with weights $w$ that should be learned as well. We usually turn it into an optimization problem to get a variance-minimizing solution, minimizing $\text{Var}(\hat{g}) = \mathbb{E}[\hat{g}^2] - \mathbb{E}[\hat{g}]^2$. Since $\hat{g}$ is unbiased, we minimize $\mathbb{E}[\hat{g}^2]$ instead, which can be further approximated by $\min_w \mathbb{E}[(f(z) - c(z;w))^2]$, indicating that the best $c(z;w)$ should be learned by fitting $f(z)$.

\subsection{Control Variates + Reparameterization Trick + Continuous Relaxation}

\begin{figure*}
\centering
\includegraphics[width=\textwidth]{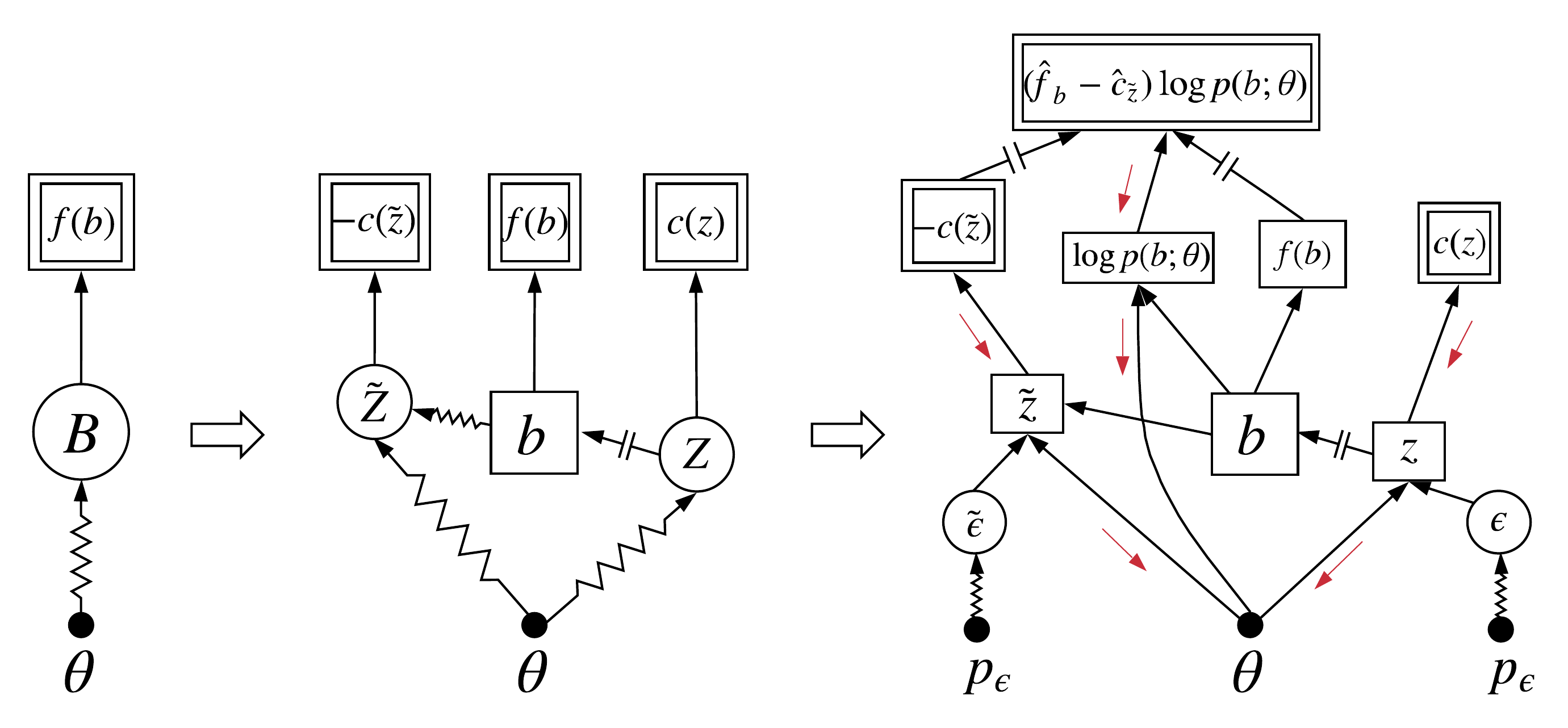}
\vspace{-5pt}
\caption{Combining control variates, reparameterization trick with continuous relaxation.}
\label{fig-cv-rt-relaxation}
\end{figure*}

The technique mentioned in the previous subsection, combining control variates with the reparameterization trick, can also be applied to discrete random variables under continuous relaxation \cite{Grathwohl2018_RELAX,Tucker2017_Rebar}. We have already seen the way to reparameterize a discrete distribution under continuous relaxation, with a temperature hyperparameter $\lambda$ tuned to control the bias. Here, we can derive an unbiased gradient estimator without the need to tune $\lambda$. The unbiasedness is guaranteed by subtracting and adding the same function as shown in Figure \ref{fig-cv-rt-relaxation}. Note that the one used as a control variate, $c(\tilde{z})$, does not have to rely on the same $z$ as in the bias-correction term $c(z)$, because we keep the total cost unbiased in the expectation level as follow:
\begin{equation}
\mathbb{E}_B [f(B)] = \mathbb{E}_{Z} \Big[ \mathbb{E}_{B|Z}\big[f(B) - \mathbb{E}_{\tilde{Z}|B} [c(\tilde{Z})] \big] + c(Z) \Big]
\end{equation}
It shows that $z$ is sampled before knowing $b$, while $\tilde{z}$ is sampled after $b$ is given. In this way, we construct a relaxed $\tilde{z}$ conditioned on $b$, so that $c(\tilde{z})$ can correlate with $f(b)$ more closely to reduce the variance. Here, $Z\sim p(\cdot;\theta)$ follows a prior distribution while $\tilde{Z}\sim p(\cdot|b;\theta)$ follows a posterior distribution, each of which is reparameterized by using a different transformation. Finally, we open up three paths to transport gradient signals back. The gradient estimator is written as:
\begin{equation}
\hat{g} := (f(b) - c(\tilde{z}))\frac{\partial}{\partial \theta}\log p(b;\theta) - \frac{\partial c}{\partial \tilde{z}}\frac{\partial}{\partial \theta}\tilde{z}(\tilde{\epsilon},b;\theta) + \frac{\partial c}{\partial z} \frac{\partial}{\partial \theta}z(\epsilon;\theta)
\end{equation}

\subsection{Gradient Estimators with Q-functions}

We have introduced several advanced approaches for gradient estimation previously. There are two ways to apply these techniques to our learned Q-functions $Q_{w_Z}(z)$ at each stochastic node $Z$ in an SCG. 

(1) We treat $Q_{w_Z}(z)$ as a local cost $f(z)$ and apply the previously introduced approaches directly. If $z$ is continuous, we use the reparameterization trick and define a transformation function $z(\epsilon;\theta)$ where $\epsilon$ is a noise from a fixed distribution. If $z$ is discrete, we introduce a relaxed random variable $z'$ following the Concrete or the Gumbel-Softmax distribution such that $z=H(z')$, and then apply control variates and the reparameterization trick, with $Q_{w_Z}(z')$ as the control variate. We also get a bias-correction term $Q_{w_Z}(\tilde{z}')$ with a different $\tilde{z}'$. However, the gradient estimator here is still biased, as $Q_{w_Z}(z)$ is an approximation to the true Q-function.

(2) From the previous two subsections, we find that $f(z)$ does not have to be a local cost. In fact, we can use an actual return from the remote cost as $f$ though very stochastic with high variance. Then, we change the role of being a local surrogate cost played by $Q_{w_Z}(z)$, and let it act as a control variate to reduce the high variance. Because $Q_{w_Z}(z)$ is learned by fitting the expectation of the remote cost, it is an ideal choice for control variates. We can also define a control variate based on $Q_{w_Z}(z)$, with a scale factor $a$ and a baseline $b$, as $c(z) =  a \cdot Q_{w_Z}(z) + b$, where $a$ and $b$ are acquired by minimizing the variance. With a second term to correct the bias, the unbiased gradient estimator is written as:
\begin{equation}
\hat{g}_{\scriptscriptstyle Z} := \big(R - a Q_{w_Z}(z) - b\big) \frac{\partial}{\partial \theta}\log p(z|{Pa}_{\scriptscriptstyle Z};\theta) + a\frac{\partial}{\partial z}Q_{w_Z}(z) \frac{\partial}{\partial \theta}z(\epsilon;\theta)
\end{equation}
where $R$ represents an actual return. If $z$ is discrete, we apply continuous relaxation the way as (1).
\section{The Big Picture of Backpropagation}

\begin{figure}
\centering
\includegraphics[width=0.7\textwidth]{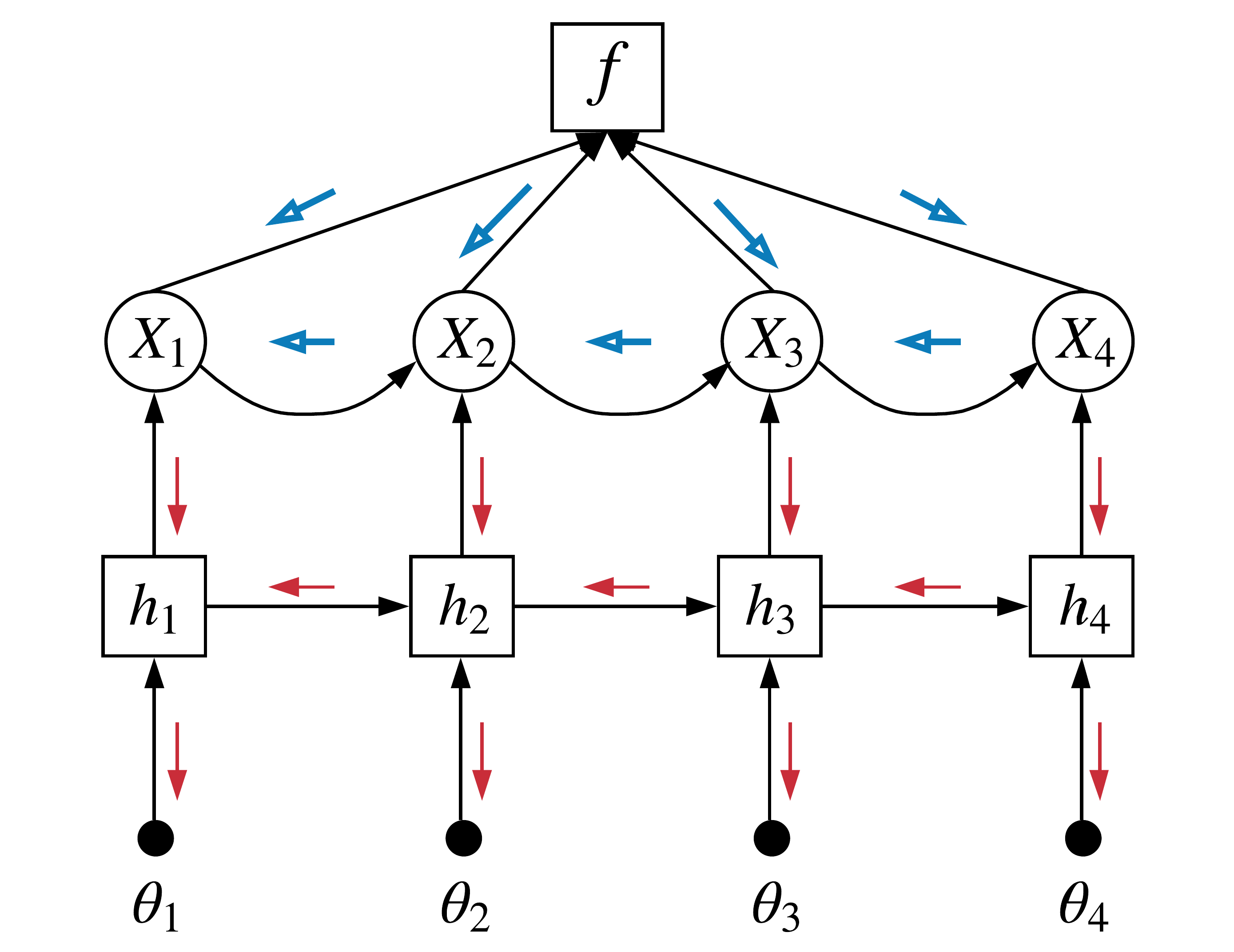}
\vspace{-5pt}
\caption{An illustration for generalized backpropagation. Blue arrows carry the signals of sample-based expected costs, and red arrows carry the signals of gradients.}
\label{fig-big-pic}
\end{figure}

Looking over the panorama of learning in a SCG, we see that the Backprop-Q framework extends backpropagation to a more general level, propagating learning signals not only across deterministic nodes but also stochastic nodes. The stochastic nodes act like repeaters, sending expected costs back over all random variables. Then, each local parameterized distribution, which is a computation subgraph consisting of many deterministic and differentiable operations, takes over the job of backpropgation and then the standard backpropgation starts. Note that these computation subgraphs can overlap by sharing the common parameters with each other. See Figure \ref{fig-big-pic}.
\section{Algorithms}

\renewcommand{\algorithmicrequire}{\textbf{Input:}}
\renewcommand{\algorithmicensure}{\textbf{Output:}}

\begin{algorithm}
\caption{Basic Framework of Backprop-Q (BPQ)}
\label{algo-bpq}
\begin{algorithmic}[1]
\REQUIRE SCG $(\mathcal{X},\mathcal{G}_{\mathcal{X}},\mathcal{P},\Theta,\mathcal{F},\Phi)$, Backprop-Q network $(\mathcal{Q},\mathcal{G}_\mathcal{Q},\mathcal{R})$, a set of approximators $\{Q_{w_X} \mid \forall Q_{\scriptscriptstyle X} \in \mathcal{Q} \}$
\STATE Initialize $(\Theta, \Phi)$ and all $w_{\scriptscriptstyle X}$
\REPEAT
\STATE // A forward pass
\FOR {each $X \in \mathcal{X}$ in a topological order of $\mathcal{G}_{\mathcal{X}}$}
\STATE Sample $x \sim p_{\scriptscriptstyle X}(\cdot|{Pa}_{\scriptscriptstyle X};\theta_{\scriptscriptstyle X})$
\STATE Store $x$ and values computed on deterministic nodes in this forward pass
\ENDFOR
\STATE Compute and store values on each cost node $f \in \mathcal{F}$
\STATE // Backpropagation across stochastic nodes
\FOR {each $Q_{\scriptscriptstyle X} \in \mathcal{Q}$ (excluding $Q_f$) in a topological order of $\mathcal{G}_{\mathcal{Q}}$}
\STATE Get sample update target ${Q}^{\text{tar}} = R^{\text{sample}}_{\scriptscriptstyle X}Q_{w_X}(Sc_{\scriptscriptstyle X})$ by applying the sample-update version of operator $R_{\scriptscriptstyle X} \in \mathcal{R}$ to approxiamtor $Q_{w_X}$ based on current samples
\STATE Take one-step SGD update on $w_{\scriptscriptstyle X}$ by: $w_{\scriptscriptstyle X} \leftarrow w_{\scriptscriptstyle X} + \alpha (Q^{\text{tar}} - Q_{w_X}(Sc_{\scriptscriptstyle X}))\nabla Q_{w_X}({Sc}_{\scriptscriptstyle X})$
\ENDFOR
\STATE // Backpropgation across deterministic nodes
\FOR {each $X \in \mathcal{X}$}
\STATE Sum over all $Q_{w_X}({Sc}_{\scriptscriptstyle X})$ to get a total local cost on $X$
\STATE Construct a local differentiable surrogate objective on $X$ using one of the gradient estimation techniques
\ENDFOR
\STATE Combine all surrogate objectives with cost functions in $\mathcal{F}$ into one
\STATE Run standard backpropagation and take one-step SCG update on $(\Theta, \Phi)$
\UNTIL $(\Theta, \Phi)$ converges
\end{algorithmic}
\end{algorithm}

\end{document}